\pdfoutput=1
\documentclass[11pt, a4paper, logo]{gdm/googledeepmind}


\pdftrailerid{redacted}

\makeatletter
\renewcommand\bibentry[1]{\nocite{#1}{\frenchspacing\@nameuse{BR@r@#1\@extra@b@citeb}}}
\makeatother

\usepackage{lipsum}
\usepackage{dsfont}
\usepackage{gdm/gdm-colors}

\usepackage[authoryear, sort&compress, round]{natbib}
\usepackage{comment}
\graphicspath{{figures/}}

\usepackage{lineno}

\usepackage{amsmath,amsfonts,bm}

\def\eqref#1{equation~\ref{#1}}

\def\1{\bm{1}}

\DeclareMathAlphabet{\mathsfit}{\encodingdefault}{\sfdefault}{m}{sl}
\SetMathAlphabet{\mathsfit}{bold}{\encodingdefault}{\sfdefault}{bx}{n}

\usepackage{amsthm}
\newtheorem{ASS}{Assumption}[section]

\newcommand\TODO[1][]{{\color{orange}[TODO\ifthenelse{\equal{#1}{}}{}{: #1}]}}
\newcommand\SEC\section
\newcommand\SSEC\subsection
\newcommand\SSSEC\subsubsection

\usepackage[colorlinks=true,allcolors=blue]{hyperref}  
\hypersetup{
citecolor=[HTML]{2420cf}
}

\hypersetup{
linkcolor=[HTML]{2420cf}
}
\usepackage{url}

\usepackage[utf8]{inputenc} %
\usepackage[T1]{fontenc}    %
\usepackage{url}            %
\usepackage{booktabs}       %
\usepackage{amsfonts}       %
\usepackage{nicefrac}       %
\usepackage{microtype}      %
\usepackage{xcolor}         %

\usepackage{bbm}
\usepackage{pifont}

\usepackage{algpseudocode}
\usepackage[ruled,vlined,linesnumbered]{algorithm2e}

\usepackage{microtype}
\usepackage{graphicx}
\usepackage{booktabs} %
\usepackage{colortbl}

\usepackage{algorithm2e}

\usepackage{amsfonts}   %
\usepackage[table]{xcolor}
\usepackage[svgnames]{xcolor}
\usepackage[most]{tcolorbox}

\usepackage{enumitem}

\usepackage{amssymb}
\usepackage{mathtools}
\usepackage{amsthm}
\usepackage{amsmath}
\usepackage{multirow}
\usepackage{makecell}
\usepackage{appendix}    
\usepackage{enumitem}
\usepackage{nccmath}
\usepackage{wrapfig}
\definecolor{darkgreen}{rgb}{0.0, 0.5, 0.0} 
\usepackage[capitalize,noabbrev]{cleveref}

\renewcommand{\eqref}[1]{Eq.~\ref{#1}}

\theoremstyle{plain}
\newtheorem{theorem}{Theorem}[section]
\newtheorem{proposition}[theorem]{Proposition}

\theoremstyle{definition}
\newtheorem{definition}[theorem]{Definition}

\theoremstyle{remark}
\newtheorem{remark}[theorem]{Remark}

\usepackage{marvosym}
\usepackage{caption}
\usepackage[most]{tcolorbox}
\usepackage{amssymb}
\usepackage{bm}

\newcommand{\ours}{\textsf{RecursiveMAS}\xspace}
\newcommand{\link}{RecursiveLink\xspace}

\newcommand{\gaincell}[1]{\cellcolor{green!66!black!20}{\textbf{#1}}}
\definecolor{impr}{RGB}{34, 139, 34}
\definecolor{lightred}{RGB}{255, 230, 230}
\definecolor{darkred}{RGB}{192, 0, 0}

\definecolor{bestbg}{HTML}{FFF2B2}   %
\definecolor{secondbg}{HTML}{E8F0FE} %
\definecolor{LavenderLight}{HTML}{C7C3F5}
\definecolor{LightCoral}{RGB}{240,128,128}
\definecolor{LightBlue}{RGB}{173,216,230}

\renewtcolorbox[auto counter, number within=section]{definition}[1]{
  enhanced,
  colback=orange!10!white,
  colframe=black!30,
  boxrule=0.6pt,
  top=7.5mm,
  bottom=.5mm,
  left=0.8mm,
  right=0.8mm,
  before skip=5pt, 
  after skip=10pt,
  fonttitle=\bfseries,
  coltitle=white,
  title={Definition~\thetcbcounter: #1},
  attach boxed title to top left={
    yshift=-7.8mm,
    xshift=1mm
  },
  boxed title style={
    colback=orange!70!black,
    rounded corners,
    boxrule=1pt,
    top=0.3pt, bottom=0.3pt, left=2pt, right=2pt
  }
}

\newtcolorbox{highlight}[1]{
  enhanced,
  colback=LightCoral!10,
  colframe=black!30,
  boxrule=0.6pt,
  top=7.5mm,
  bottom=.5mm,
  left=0.8mm,
  right=0.8mm,
  before skip=5pt,
  after skip=10pt,
  fonttitle=\bfseries,
  coltitle=white,
  title={#1},
  attach boxed title to top left={
    yshift=-7.8mm,
    xshift=1mm
  },
  boxed title style={
    colback=FireBrick,
    rounded corners,
    boxrule=1pt,
    top=0.3pt,
    bottom=0.3pt,
    left=2pt,
    right=2pt
  }
}

\tcolorboxenvironment{theorem}{
    colback=LightBlue!20,       %
    colframe=black,     %
    arc=1mm,               %
    top=.6mm,
    bottom=.6mm,
    right=1mm,
    boxrule=.5pt,           %
    fonttitle=\bfseries,   %
    left=1mm,              %
    attach boxed title to top left={yshift=-2mm, xshift=4mm},
    boxed title style={
        colback=black!40,
        arc=1mm,
    },
    before skip=6pt,   %
    after skip=6pt,    %
}

\tcolorboxenvironment{proposition}{
    colback=LightBlue!20,       %
    colframe=black,     %
    arc=1mm,               %
    top=.6mm,
    bottom=.6mm,
    right=1mm,
    boxrule=.5pt,           %
    fonttitle=\bfseries,   %
    left=1mm,              %
    attach boxed title to top left={yshift=-2mm, xshift=4mm},
    boxed title style={
        colback=black!40,
        arc=1mm,
    },
    before skip=6pt,   %
    after skip=6pt,    %
}

\tcolorboxenvironment{lemma}{
    colback=red!10,       %
    colframe=black!,     %
    arc=1mm,               %
    boxrule=1pt,           %
    fonttitle=\bfseries,   %
    left=1mm,              %
    attach boxed title to top left={yshift=-2mm, xshift=4mm},
    boxed title style={
        colback=black!40,
        arc=1mm,
    }
}

\tcolorboxenvironment{remark}{
    colback=LightCoral!10,       %
    colframe=black,     %
    arc=1mm,               %
    top=.6mm,
    bottom=.6mm,
    boxrule=.5pt,           %
    right=1mm,
    fonttitle=\bfseries,   %
    left=1mm,              %
    attach boxed title to top left={yshift=-2mm, xshift=4mm},
    boxed title style={
        colback=black!40,
        arc=1mm,
    },
    before skip=6pt,   %
    after skip=6pt,    %
}

\title{Recursive Multi-Agent Systems}

\correspondingauthor{jiaru@stanford.edu}

\renewcommand{\today}{}

\author[1 $\dagger$]{Jiaru Zou}
\author[2]{Rui Pan}
\author[2]{Ruizhong Qiu}
\author[1]{Pan Lu}
\author[3]{Shizhe Diao}
\author[3]{Jindong Jiang}
\author[2]{Hanghang Tong}
\author[2]{Tong Zhang}
\author[4]{Markus J. Buehler}
\author[2 \Letter]{Jingrui He}
\author[1 \Letter]{James Zou}

\affil[1]{Stanford University }
\affil[2]{University of Illinois Urbana-Champaign }
\affil[3]{NVIDIA }
\affil[4]{MIT \newline}
\affil[$\dagger$]{Project Lead }
\affil[\Letter]{Corresponding Authors}

\begin{abstract}
\vspace{-18pt}
{\fontsize{10pt}{10pt} \selectfont \raisebox{-0.1em}{\includegraphics[height=1em]{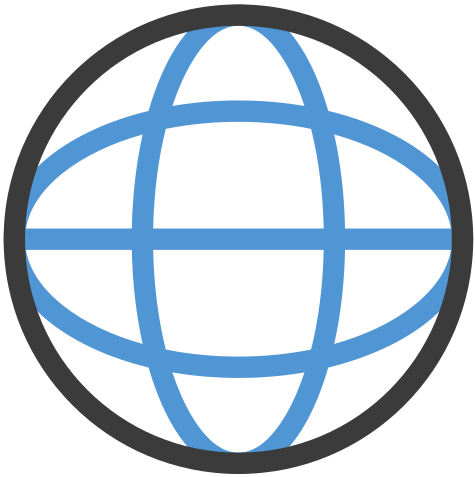}} Project Page: \href{https://recursivemas.github.io}{https://recursivemas.github.io}}
\\[0.6em]
Recursive or looped language models have recently emerged as a new scaling axis by iteratively refining the same model computation over latent states to deepen reasoning.
We extend such scaling principle from a single model to multi-agent systems, and ask: 
\textit{Can agent collaboration itself be scaled through recursion?} 
To this end, we introduce \textsf{RecursiveMAS}, a recursive multi-agent framework that casts the entire system as a unified latent-space recursive computation.
\textsf{RecursiveMAS} connects heterogeneous agents as a collaboration loop through the lightweight RecursiveLink module, enabling in-distribution latent thoughts generation and cross-agent latent state transfer.
To optimize our framework, we develop an inner-outer loop learning algorithm for iterative whole-system co-optimization through shared gradient-based credit assignment across recursion rounds.
Theoretical analyses of runtime complexity and learning dynamics establish that \textsf{RecursiveMAS} is more efficient than standard text-based MAS and maintains stable gradients during recursive training.
Empirically, we instantiate \textsf{RecursiveMAS} under 4 representative agent collaboration patterns and evaluate across 9 benchmarks spanning mathematics, science, medicine, search, and code generation.
In comparison with advanced single/multi-agent and recursive computation baselines, \textsf{RecursiveMAS} consistently delivers an average accuracy improvement of 8.3\%, together with 1.2$\times$--2.4$\times$ end-to-end inference speedup, and 34.6\%--75.6\% token usage reduction.

\end{abstract}

\begin{document}

\maketitle

\begin{figure}[!h]
    \vspace{-5pt}
    \centering
    \includegraphics[width=\linewidth]{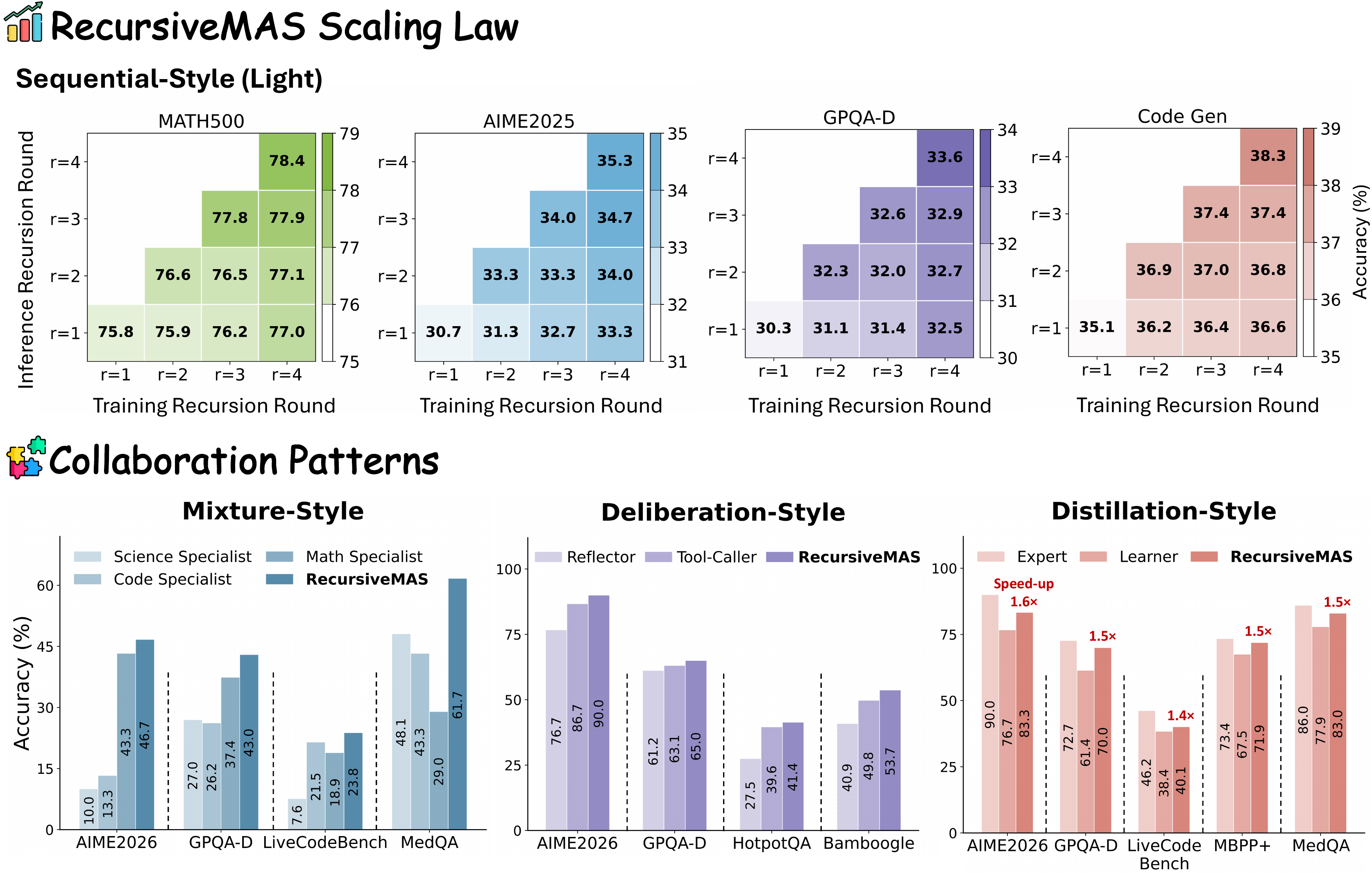}
    \vspace{-20pt}
    \caption{
    \textbf{Performance Landscape of \ours across Training/Inference Recursion Depths (Top):} The \textit{lightweight} \ours with sub-1.5B agents shows a clean scaling trend as recursion deepens.
    \textbf{Generalization across Common Collaboration Patterns (Bottom):} 
    The \textit{Scaled} \ours with stronger LLM agents (5-10B) seamlessly adapts to diverse multi-agent system structures.
    }
    \label{fig:overall}
    \vspace{-40pt}
\end{figure}

\addtocontents{toc}{\protect\setcounter{tocdepth}{-1}}
\newpage
\section{Introduction}

To tackle complex tasks, a single language model often falls short due to limited capacity, myopic generation, or inefficient exploration of the solution space \citep{shojaee2025illusion, song2026large, li2025system}. 
Once intelligence reaches a threshold, a natural direction is to treat individual models as specialized agents and organize them as a collaborative system \citep{tran2025multi, xu2025towards}. A multi-agent system (MAS) \citep{wang2025mixture, wu2024autogen} can scale performance by enabling individuals to work together and contribute complementary strengths.
Consider a set of heterogeneous agents, each assigned a distinct role and expertise. The system can either arrange agents into a sequential pipeline \citep{gu2025explain, qian2024chatdev} to progressively decompose and solve a problem, or engage and integrate multiple domain-specialized agents \citep{ye2025x, qian2025scaling, babu2025adaptive} for the task.

While MAS establishes a structural foundation, the next question is how to enable the system to evolve over time and adapt to different scenarios.
Prior work has explored prompt-based adaptation \citep{shen2025optimizing, zhou2025multi, zhang2025agentic}, where model interactions are improved through the iterative refinement of shared context. 
Although these updated prompts can help agents generate more aligned responses to the question, each agent itself cannot improve. 
A more principled line of work is to optimize agents through learning \citep{motwani2024malt, subramaniam2025multiagent, zhao2025sirius}. 
However, training entire agents inside the system is hard, as updating all model parameters is non-trivial \citep{hu2025owl}, and the sequential dependency in text-based interactions introduces substantial latency when agents must wait for others to complete generation.

Instead of improving each agent's capabilities as a standalone component, we adopt a higher-level learning perspective and aim to co-evolve and scale the entire system as an integrated whole. 
We recast agent collaboration through the lens of recursive language models (RLMs) \citep{zhang2025recursive, jolicoeur2025less,zhu2025scaling}, where a shared set of layers is iteratively applied and optimized within a continuous latent space. 
In this view, the entire multi-agent system can be treated as a recursive computation, where each agent acts like an RLM layer, iteratively passing latent representations to the next and forming a looped interaction process.

We call this new system-level agentic recursion framework \textbf{\ours}. Without updating all model parameters, agents are connected and iteratively optimized solely via the lightweight \textbf{\link}, a two-layer residual projection module for latent states transmission and refinement.
An \textit{inner \link} within each agent first consolidates the model's ongoing latent thoughts between input and output spaces during auto-regressive generation. 
An \textit{outer \link} then bridges hidden representations across heterogeneous agents built on different model types and sizes, enabling seamless cross-agent interaction. Together, all agents are chained in a unified loop to perform iterative latent collaboration, with only the last agent producing the textual output in the final recursion round.

Correspondingly, we pair \ours with an \textbf{Inner-Outer Loop} training paradigm for progressive co-optimization.
The \textit{inner loop} provides a preliminary model-level warm start for each agent, by training its inner \link to better align with latent thoughts generation.
The \textit{outer loop} then trains the outer \link across agents at the system-level, with gradients recursively back-propagated through the full computation traces over recursion rounds.
By exposing each agent to the feedback of itself and others from previous rounds, \ours learns to leverage \link for iterative refinement of collaboration, thus enabling the entire system to optimize in a unified manner.

To justify why recursion should occur in latent space rather than text-mediated interaction, we provide two theoretical analyses on runtime complexity and learning dynamics. From an architectural standpoint, RecursiveLink enables direct transformation of latent-space information, avoiding repeated decoding of intermediate agents with more efficient runtime complexity.
From the learning perspective, latent-space connections in \ours maintain stable gradient propagation flow across recursion rounds during training, avoiding the gradient vanishing induced by text-based interactions.

Empirically, we evaluate \ours on 9 benchmarks spanning mathematics, science, medicine, search, and code generation. We instantiate \ours with diverse model families, including Qwen3/3.5, LLama-3, Gemma3, and Mistral, and adapt our framework to 4 representative MAS collaboration scenarios: step-by-step sequential reasoning, mixture-of-experts collaboration, expert-to-learner knowledge distillation, and tool-integrated deliberation. 
As illustrated in Figure \ref{fig:overall}, compared with advanced recursive language models and MAS baselines, \ours achieves an average accuracy improvement of 8.3\%, while delivering 1.2$\times$--2.4$\times$ inference speedup and reducing token usage by 34.6\%--75.6\%. 
In addition, \ours is structure-agnostic and can generalize to various agent collaboration patterns with effective performance.
Our additional detailed analyses of scaling laws with deeper recursion, \link architectures, semantic distributions across recursions, and training cost further validate the efficiency and performance scalability of the \ours.

\vspace{-10pt}
\section{Preliminary}
\label{sec:preliminary}
\vspace{-5pt}

\textbf{Auto-regressive Generation in Latent Space.}
Let $f_\theta(\cdot)$ denote a standard Transformer model \citep{ashish2017attention} parameterized by $\theta$. Given a question $x$ with corresponding input embeddings $E=[e_1, \ldots, e_t]\in \mathbb{R}^{t \times d_h}$, the model computes the last-layer hidden state $h_t$ through the forward pass. 
In standard auto-regressive decoding, $h_t$ is projected to the vocabulary space to predict the next token. 
In contrast, latent generation keeps the recurrence entirely in continuous representation space by directly feeding the previously generated latent embedding $h_t$ back into the next forward pass. Formally, the next latent generation at step $t+1$ is:
\begin{equation}
h_{t+1} = f_{\theta}([E_{\le t};\, h_t]).
\end{equation}
We refer to the newly generated latent state $h_{t+1}$ as the model's ongoing \textit{latent thought}.

\textbf{Recursive Computation.} A recursive language model (RLM) increases reasoning depth by reusing the same transformation across recurrent steps. Consider a Transformer $f_\theta$ with $L$ layer blocks, denoted as $f_\theta = \mathcal{M}_L \circ \cdots \circ \mathcal{M}_1$. Instead of passing the input through the $L$-layer stack only once to obtain the last representation, a recursive model reuses the same stack for $n$ times of forward iterations, i.e.,
\begin{equation}
 H^{(0)}=E, \quad H^{(r)} = f_\theta\big(H^{(r-1)}\big), \quad r=1,\dots,n.
\end{equation}
The last round of latent representation $H^{(n)}$ is obtained through recursive refinement over the same shared Transformer layers, and is subsequently used for the final prediction.

\textbf{LLM-based Multi-Agent Evolution.} We define a multi-agent system $\mathcal{S}$ \citep{tran2025multi,zou2025latentmas} composed of $N$ agents denoted as $\mathcal{A} = \{A_1, \dots, A_N\}$, where each LLM agent $A_i$ corresponds to $f_{\theta_i}$ with its own last-layer representations $H_i$. We then denote the collective latent state of the system by $\mathcal{H} = \{H_1, \dots, H_N\}$. Given any input problem $x$ with the ground-truth $y$, the system $\mathcal{S}$ orchestrates interactions among agents to collaboratively produce a final prediction. With this setup in place, we now formalize the evolution of agents under recursive computation. 

\begin{definition}{Recursive Multi-Agent Evolution}
\textit{A \textbf{recursive evolution} is the progressive refinement of $\mathcal{H}$, where each agent adjusts its latent representation through iterative interaction with others and its own reasoning state, so that the updated system is better aligned for the given problem, i.e. 
$
\mathcal{S}^{(0)}
\xrightarrow[\text{Evolve}]{H^{(1)}}
\mathcal{S}^{(1)}
\xrightarrow[\text{Evolve}]{H^{(2)}}
\cdots
\xrightarrow[\text{Evolve}]{H^{(n)}}
\mathcal{S}^{(n)}
$.
}
\end{definition}

\textbf{Collaboration Pattern.} As MAS architectures are generally not fixed and can vary across tasks, we do not restrict the collaboration pattern to a single style. In this paper, we consider four commonly adopted collaboration patterns in multi-agent systems: (i) \textit{Sequential Style}, where we follow the chain-of-agents setting to assign three agents with complementary roles of \texttt{Planner}, \texttt{Critic}, and \texttt{Solver} and progressively decompose, judge, refine, and solve the problem; (ii) \textit{Mixture Style}, where a mixture of domain-specialized agents (\texttt{Math}, \texttt{Code}, \texttt{Science}) reasons over the input problem in parallel, and their outputs are aggregated by a \texttt{Summarizer} agent to form the final answer; (iii) \textit{Distillation Style}, where a larger, more capable \texttt{Expert} agent is paired with a smaller, faster \texttt{Learner} agent to distill expert knowledge while retaining higher generation efficiency; and (iv) \textit{Deliberation Style}, where an inner-thinking \texttt{Reflector} is paired with a \texttt{Tool-Caller} that can invoke external tools (e.g., Python or search APIs). The agents iteratively exchange, critique, and refine candidate solutions until reaching a shared consensus, after which the \texttt{Tool-Caller} produces the final answer.

\vspace{-10pt}
\section{Building a Recursive Multi-Agent System}
\vspace{-5pt}

\begin{figure}[!t]
    \centering
    \includegraphics[width=\linewidth]{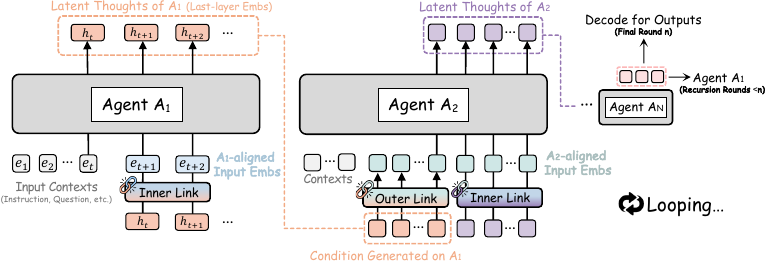}
    \caption{\textbf{Overall Architecture of \ours.} Each agent first leverages the inner \link to perform latent thoughts generation, and then transfers the generated information to the next agent through the outer \link. After the last agent finishes generation, its latent thoughts are fed back to the first agent, thereby forming a recursive loop within the multi-agent system.}
    \label{fig:recursivemas}
    \vspace{-10pt}
\end{figure}

We introduce \ours, an end-to-end recursive framework that links heterogeneous LLM agents together to scale the entire system through efficient and seamless latent collaboration. In the following, we will first elaborate the detailed architectural design of \ours, and then present the corresponding recursive learning algorithm. We also interleave theoretical analyses throughout the method pipeline to support underlying design principles.

\vspace{-10pt}
\subsection{A Lightweight \link}
\label{sec:recursive_link}

\begin{wrapfigure}{r}{0.4\textwidth}
    \vspace{-20pt}
    \centering
    \includegraphics[width=0.4\textwidth]{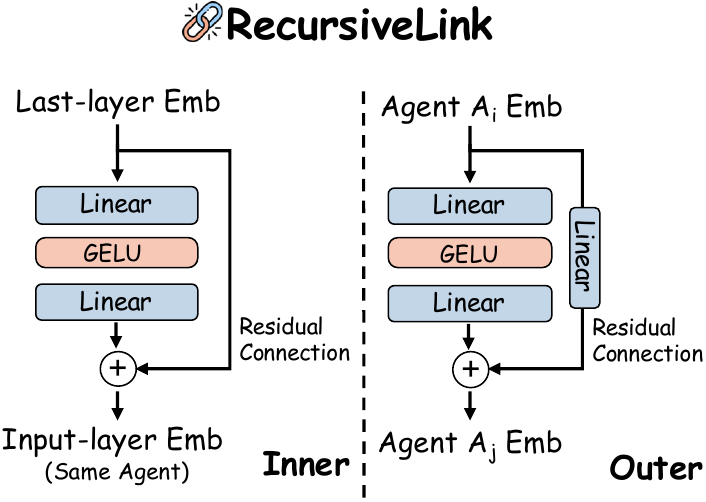}
    \caption{Illustration on the inner and outer \link Design.}
    \label{fig:recursivelink}
    \vspace{-15pt}
\end{wrapfigure}
A language model's last-layer hidden states provide a natural representation of its generated semantics.
The \link $\mathcal{R}$ is designed to preserve and transmit this information from one embedding space to another.
In \ours, the transition arises in two cases:
(i) \textit{Dense-to-Shallow Transition}, where the previous step's last-layer embeddings are fed back as the next-step input embeddings during latent thoughts generation;
and (ii) \textit{Cross-Model Transition}, where one model's newly generated latent representations are passed as conditioning inputs to another model.
As illustrated in Figure \ref{fig:recursivelink}, we bridge these two transitions through the inner and outer links.

\textbf{Inner Link.} Each LLM agent $A_i \in \mathcal{A}$ is paired with an inner \link $\mathcal{R}_\text{in}$ during auto-regressive generation. Given any new last-layer embedding vector $h$, $\mathcal{R}_\text{in}$ transforms it as:
\begin{equation}
\mathcal{R}_\text{in}(h) = h + W_2\,\sigma(W_1 h),
\end{equation}
where $W_1$ and $W_2$ are two standard linear layers, $\sigma(\cdot)$ is the GELU activation, and the residual connection preserves the original latent semantics. The transformed embedding is then used as input to the next forward pass of agent $A_i$. 

\textbf{Outer Link.} An outer \link $\mathcal{R}_{\mathrm{out}}$ connects heterogeneous agents with different hidden dimensions. To support this, an additional linear layer $W_3$ is introduced in the residual branch to map the source embedding from agent $A_i$ into the target embedding space of agent $A_j$, i.e.,
\begin{equation}
\mathcal{R}_{\mathrm{out}}(h)=W_3h+W_2\,\sigma(W_1h).
\end{equation}

\begin{highlight}{Why Residual Connection?}
The residual branch largely preserves the original semantics of the input, allowing the \link network to focus on \textit{aligning distributional differences} rather than learning the full projection from scratch. This leads to more stable and efficient training. We also explore other alternatives and empirically validate our proposed design in Section \ref{sec:experiments}.
\end{highlight}

\subsection{Chain All Agents Together as a Loop}

In recursive language models (RLMs), Transformer layers are connected through hidden states, and the residual stream loops across these layers to increase reasoning depth.
Under this view, we cast each agent in \ours as an RLM layer, with information flowing and recurring within and across agents as the hidden stream of the system. As shown in Figure \ref{fig:recursivemas}, each agent contributes by reasoning and interacting with others in the latent space, together forming a recursive loop.

\textbf{Latent Thoughts Generation inside Agents.} We start by describing how each agent unfolds reasoning through the auto-regressive generation of latent thoughts.
Specifically, given input contexts' embeddings $E_{A_1} = [e_1, e_2, \dots, e_t]$ for the question and the agent-specific instructions, the first agent $A_1$ passes $E_{A_1}$ through the Transformer and computes the last-layer hidden representation $h_t$ at step $t$. 
Then, we insert $h_t$ into the inner link $\mathcal{R}_{\mathrm{in}}$ to map the distribution back into the input embedding space for the next step, yielding $e_{t+1} = \mathcal{R}_{\mathrm{in}}(h_t)$. Agent $A_1$ repeats this process auto-regressively for $m$ forward steps, generating a new continuous sequence of latent thoughts $H_{A_1} = [h_{t}, h_{t+1}, \dots, h_{t+m}]$.

\textbf{Interaction across Heterogeneous Agents.} Once agent $A_1$ completes latent reasoning, its latent thoughts $H_{A_1}$ are sent to the next agent $A_2$ for cross-agent interaction. To achieve seamless information transmission across different types of agents, we first pass $H_{A_1}$ through the outer link $\mathcal{R}_\text{out}$ to transform it into input embeddings aligned with agent $A_2$. Next, agent $A_2$ starts latent thoughts generation conditioned on both its own input contexts and transferred information from $A_1$ (i.e., $E_{A_2} \oplus \mathcal{R}_\text{out}(H_{A_1})$).

We continue this interaction process across all consecutive agents in \ours. In particular, after the last agent $A_N$ completes latent thoughts generation, its latent outputs (representing the system's latent answer to the input question) are passed back to the first agent $A_1$ through the inner-outer \link, thereby closing the recursive loop. This recurrent connection allows each new recursion round to condition on information produced in previous rounds, so that each agent can iteratively reflect on earlier system outputs and refine their current generation.
Throughout intermediate recursion rounds, all agents collaborate entirely in the latent space. Only after the final recursion round, the agent $A_N$ decodes the textual output as the system's final answer to the question.

\textbf{End-to-End Complexity Analyses.}
To characterize the architectural efficiency of the full \ours pipeline, we next analyze its end-to-end runtime complexity with \link integrated throughout the system. The following proposition compares \ours with a text-based recursive MAS, in which agents follow the same multi-round recursive collaboration structure but communicate through an explicit text medium rather than \link-enabled latent interaction.

\begin{proposition}[\textbf{\ours Runtime Complexity}]
\label{prp:eff}
Without \link, a text-based Recursive MAS with the same collaboration structure requires runtime complexity of $\Theta(N(m|V|d_h+(t+m)d_h^2+(t+m)^2d_h))$; In contrast, with \link-enabled collaboration, \ours{} achieves an end-to-end runtime complexity of $\Theta(N(md_h^2+(t+m)d_h^2+(t+m)^2d_h))$. 
\end{proposition}

\begin{remark}
Since \(d_h \ll |V|\) in practice, \ours replaces the expensive per-step vocabulary-space decoding cost $m|V|d_h$ with a much more efficient latent-space transformation $md_h^2$.
\end{remark}

Proposition~\ref{prp:eff} shows the end-to-end runtime advantage of \ours. The full proof is provided in Appendix~\ref{app:proposition_proof}. We also empirically analyze the efficiency advantage of our method in Section~\ref{sec:experiments}.

\vspace{-10pt}
\section{Learning to Recur as a Whole}

\begin{figure}[!t]
    \centering
    \includegraphics[width=\linewidth]{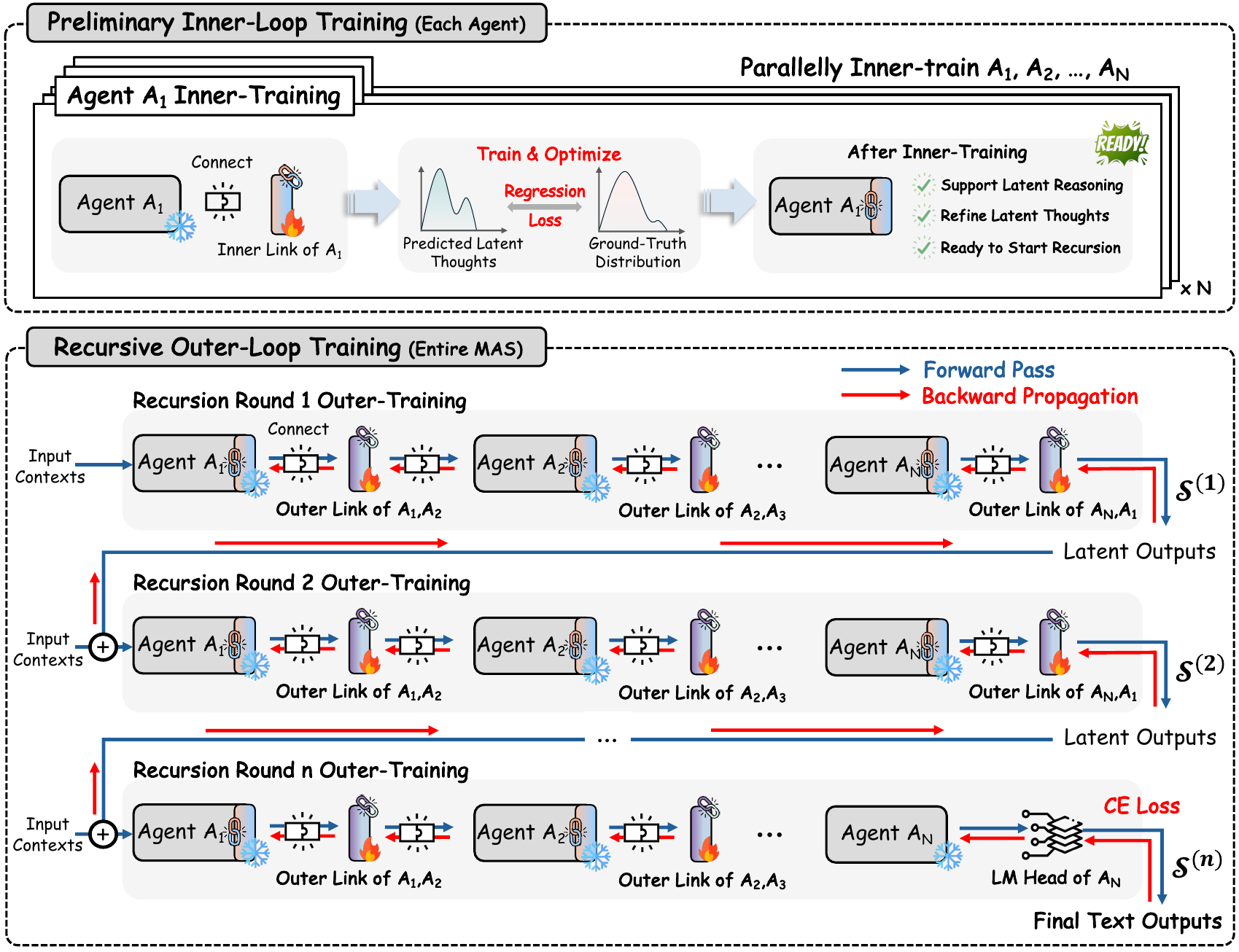}
    \caption{\textbf{Two-Stage Training Pipeline of \ours.} We first perform inner-loop training for each agent in parallel to warm up the inner \link for latent thoughts generation, and then conduct outer-loop training to recursively optimize the outer \link over the entire system.}
    \label{fig:recursivetraining}
    \vspace{-10pt}
\end{figure}

With the framework in place, we next present the recursive learning algorithm, which only needs to train on the \link to enable co-optimization of the entire system loop. As illustrated in Figure \ref{fig:recursivetraining}, the learning procedure consists of two stages: (i) a preliminary \textit{inner-loop} to equip each agent with stronger latent thoughts generation capabilities; and (ii) an iterative \textit{outer-loop} to progressively optimize the system as one unified entity over recursion rounds.

\textbf{Model-Level Inner-Loop Training.} For practical deployment of \ours, we directly adopt off-the-shelf text-generation models as agents. To adapt these agents to the latent thoughts generation pattern, we first warm-start them through the inner \link $\mathcal{R}_{\mathrm{in}}$.
Specifically, given each agent $A_i \in \mathcal{A}$ with parameters $\theta_i$ and the training example $(x, y) \in \mathcal{D}_{\text{train}}$, we construct the target latent thoughts distribution by passing the ground-truth text $y$ through the standard input embedding layer $\mathrm{Emb}_{\theta_i}$ of agent $A_i$. The objective of training the inner link $\mathcal{R}_{\mathrm{in}}$ corresponding to $A_i$ then formulates as:
\begin{equation}
\label{eq:inner_train}
    \mathcal{L}_\mathrm{in} = 1 - \cos\big(\mathcal{R}_{\mathrm{in}}(H), \mathrm{Emb}_{\theta_i}(y)\big),
\end{equation}
where $H$ denotes the last-layer latent thoughts generated by agent $A_i$, and $\cos(\cdot,\cdot)$ denotes the standard cosine similarity.
The regression objective here encourages each agent to leverage its inner link $\mathcal{R}_{\mathrm{in}}$ to align latent thoughts with the semantic distribution from the input embedding layer, while eliminating the process of explicit decoding and re-encoding.

\textbf{System-Level Outer-Loop Training.} Next, we iteratively co-optimize the entire system through the outer \link $\mathcal{R}_{\mathrm{out}}$. Let $\mathcal{S}^{(r)}$ denote the system state at recursion round $r = 1, \dots, n$. During outer-loop training, the system is first unrolled along its looped structure for $n$ forward rounds. 
After the final textual prediction is produced in the last recursion round, we jointly optimize all outer links that connect the system with the following cross-entropy (CE) objective:
\begin{equation}
\mathcal{L}_{\mathrm{out}} = \mathrm{CE}\!\left(
\mathcal{S}^{(n)}\!\bigl(\mathcal{S}^{(n-1)}(\cdots \mathcal{S}^{(1)}(x))\bigr),\, y
\right).
\end{equation}
Throughout training, the computation graph is preserved along the full recursive paths. Gradient backpropagation assigns each outer link a shared credit signal according to its global contribution to the final prediction, thereby enabling information flow to be iteratively optimized as a whole.

\textbf{Learning Advantage of \ours.}
To better understand why latent collaboration of agents in the inner-outer loop training confers a stronger learning advantage, we provide a detailed theoretical analysis below of the gradient propagation process throughout recursive training of \ours.

\begin{theorem}[\textbf{Gradient Stability}]
\label{thm:grad}

Under the Realistic Assumptions (stated in Appendix \ref{app:assumptions}), if tokens are confident with entropy $\le \epsilon$, where typically $\epsilon \ll 1$: directly applying text-based SFT (denoted by $\mathcal R_\text{text}(h)$) during recursion suffers from gradient vanishing (i.e., gradient norm close to 0); while \ours with the \link $\mathcal{R}$ maintains stable and near constant gradients (i.e., gradient norm close to 1) during looped backpropagation process. Formally, with probability $\ge 1-\delta$,

\begin{align}
\left\|\frac{\partial\mathcal R_\text{text}(h)}{\partial h}\right\|_2\le O(\epsilon)\ll1,\qquad\left\|\frac{\partial\mathcal R(h)}{\partial h}\right\|_2\ge\Omega\left(1-\sqrt{\frac1{d_h}\log\frac1{\delta}}\right).%
\end{align}
\end{theorem}

The full proof is provided in Appendix \ref{app:proof-grad}. Theorem \ref{thm:grad} demonstrates the learning advantage of \ours, by allowing gradients to remain informative across recursion rounds.
Together, theoretical justifications in Proposition \ref{prp:eff} and Theorem \ref{thm:grad} motivate our design of latent-based interaction among agents rather than text mediation, as it makes the whole-system co-optimization of \ours easier and more effective.
During inference, \ours performs recursive generation by following the same $n$ recursion rounds as in the outer-loop training.

\section{Empirical Evaluations}
\label{sec:experiments}

\textbf{Tasks and Datasets.} 
We conduct comprehensive evaluations of \ours on nine benchmarks across various domains: (i) \textit{Mathematical Reasoning}, including MATH500 \citep{math500}, AIME2025 \citep{aime25}, and AIME2026 \citep{aime26}; (ii) \textit{Scientific and Medical Tasks}, including GPQA-Diamond \citep{gpqa} and MedQA \citep{medqa}; (iii) \textit{Code Generation}, including LiveCodeBench-v6 \citep{lcb} and MBPP Plus \citep{codeplus}; and (iv) \textit{Search QA}, including HotpotQA \citep{yang2018hotpotqa} and Bamboogle \citep{Bamboogle}.
We adopt the standard evaluation metric for each dataset. For AIME2025/2026, we report Pass@10 accuracy for testing robustness. Additional benchmark and metrics details are in Appendix~\ref{app:datasets}.

\begin{table}[!t]
\centering
\small
\caption{\textbf{Agent configurations for different collaboration patterns in \ours.}
We select off-the-shelf models from diverse model families to form heterogeneous agent compositions with complementary strengths. Each assignment is chosen to match the role-specific needs of the corresponding collaboration pattern while preserving both practical efficiency and scalability.
}
\label{tab:mas_config}
\resizebox{\textwidth}{!}{
\begin{tabular}{l|l|l}
\toprule
\textbf{Collaboration Pattern} & \textbf{Role} & \textbf{Model Size \& Version} \\
\midrule

\multirow{3}{*}{Sequential Style (Light)}
& Planner & Qwen3-1.7B \citep{yang2025qwen3} \\
& Critic & Llama3.2-1B-Instruct \citep{grattafiori2024llama} \\
& Solver & Qwen2.5-Math-1.5B-Instruct \citep{qwen2.5} \\
\midrule

\multirow{3}{*}{Sequential Style (Scaled)}
& Planner & Gemma3-4B-it \citep{gemma3}\\
& Critic & Llama3.2-3B-Instruct \citep{grattafiori2024llama} \\
& Solver & Qwen3.5-4B \citep{yang2025qwen3} \\
\midrule

\multirow{4}{*}{Mixture Style}
& Code Specialist & Qwen2.5-Coder-3B-Instruct \citep{hui2024qwen2} \\
& Science Specialist & BioMistral-7B \citep{labrak2024biomistral} \\
& Math Specialist & DeepSeek-R1-Distill-Qwen-1.5B \citep{qwen2.5} \\
& Summarizer & Qwen3.5-2B \citep{yang2025qwen3} \\
\midrule

\multirow{2}{*}{Distillation Style}
& Learner & Qwen3.5-4B \citep{yang2025qwen3}\\
& Expert & Qwen3.5-9B \citep{yang2025qwen3} \\
\midrule
\multirow{2}{*}{Deliberation Style}
& Reflector & Qwen3.5-4B \citep{yang2025qwen3} \\
& Tool-Caller & Qwen3.5-4B (with Tool-Integration) \citep{yang2025qwen3} \\

\bottomrule
\end{tabular}
}
\end{table}

\textbf{Models and Baselines.}
We instantiate \ours\ with diverse agent collaboration patterns, including (i) \textit{Sequential Style}, (ii) \textit{Mixture Style}, (iii) \textit{Distillation Style}, and (iv) \textit{Deliberation Style}, following the setups described in Section~\ref{sec:preliminary}. For each collaboration style, we use off-the-shelf LLMs from diverse model families, covering Qwen \citep{qwen2.5,yang2025qwen3}, Llama \citep{grattafiori2024llama}, Gemma \citep{gemma3}, and Mistral \citep{jiang2024mixtral}, to construct heterogeneous agent compositions. Detailed model configurations and their assigned roles are provided in Table~\ref{tab:mas_config}.

For baseline comparisons, we evaluate \ours against (i) \textit{Single Advanced Agents}, where individual LLM agents from each collaboration pattern are isolated as standalone models to solve problems, such as the final agent in Sequential Style and each domain specialist in Mixture Style. For fair comparison, we provide full supervised and LoRA fine-tuning \citep{schulman2025lora} for single models on the same training set.
(ii) \textit{Recursion-based Methods}, including single recursive language models, LoopLM \citep{zhu2025scaling}, and Recursive-TextMAS, where agents collaborate in the same way as \ours but interact through text instead of latent thoughts;
and (iii) additional \textit{Representative Multi-Agent Frameworks}, including TextGrad \citep{yuksekgonul2025optimizing} and Mixture-of-Agents (MoA) \citep{wang2025mixture} for more holistic structure-wide evaluations. Detailed baseline implementations are provided in Appendix~\ref{app:baseline}.

\textbf{Training and Implementation Details.} For inner-outer loop training, we freeze all LLM agent parameters and update only the inner/outer \link. We curate a diverse training set spanning multiple domains, sourced from s1K \citep{s1} for mathematical problem solving, m1k \citep{m1} for medical and scientific tasks, OpenCodeReasoning \citep{opencode} for code generation, and ARPO-SFT \citep{dong2025agentic} for agentic tool-augmentation (Python Code/Search-API) settings.
We use AdamW with a learning rate of 5e-4, a cosine learning rate scheduler, and a batch size of 4. During inference, we set top-p to 0.95 and use a temperature of 0.6 for most reasoning tasks and $0.2$ for code generation, as suggested in each model's official report. The maximum output length is adjusted for each task based on its relative difficulty. We perform hyperparameter tuning and report the mean performance over five independent runs. More training/inference details and hyperparameter setups are provided in Appendix~\ref{app:train}.

\begin{table}[!t]
\centering
\small
\caption{\textbf{Main results of \ours over Different Recursion Rounds.}
We report the accuracy (\%, ``Acc.''), end-to-end runtime (s, ``Time''), and overall token usage (``Token'') across domains. 
For Code Gen., we evaluate the Light and Scaled settings on MBPP+ and LiveCodeBench, respectively.
The average standard deviation of \ours across 5 runs is $\pm0.0041$ for accuracy, $\pm26$ for runtime, and $\pm33$ for tokens.
We compare with all methods under the same MAS framework structure and recursion budgets.
The performance and efficiency advantages of \ours become increasingly significant as the recursion round $r$ increases, with improvements highlighted.
}
\label{tab:seq_mas}
\vspace{-5pt}
\resizebox{\textwidth}{!}{
\begin{tabular}{l|c|cccccc|cccc|cc|c}
\toprule

\multirow{2}{*}{\makecell[l]{\textbf{Method} }}
& \multirow{2}{*}{\textbf{Metric}} 
& \multicolumn{2}{c}{\textbf{Math500}} 
& \multicolumn{2}{c}{\textbf{AIME2025}} 
& \multicolumn{2}{c|}{\textbf{AIME2026}} 
& \multicolumn{2}{c}{\textbf{GPQA-D}} 
& \multicolumn{2}{c|}{\textbf{MedQA}}
& \multicolumn{2}{c|}{\textbf{Code Gen.}} 
& \multirow{2}{*}{\textbf{Improve}} \\

\cmidrule(lr){3-4} \cmidrule(lr){5-6} \cmidrule(lr){7-8} \cmidrule(lr){9-10} \cmidrule(lr){11-12} \cmidrule(lr){13-14}
& 
& \text{Light} & \text{Scaled}
& \text{Light} & \text{Scaled}
& \text{Light} & \text{Scaled}
& \text{Light} & \text{Scaled}
& \text{Light} & \text{Scaled}
& \text{Light} & \text{Scaled}
& \\

\specialrule{\lightrulewidth}{0pt}{0pt}
\multicolumn{15}{>{\columncolor{gray!30}}c}{\textit{\textbf{Recursive Round r=1}}} \\
\specialrule{\lightrulewidth}{0pt}{0pt}

\multirow{3}{*}{\text{Recursive-TextMAS}} 
& Acc.   & 71.9 & 84.2 & 24.0 & 71.3 & 16.7 & 76.7 & 28.1 & 61.5 & 29.0 & 76.1 & 30.7 & 38.5 & \cellcolor{gray!0}\textcolor{gray}{Base} \\
& Time   & 1368 & 2401 & 2380 & 8462 & 2216 & 9376 & 1056 & 2190 & 1555 & 1522 & 976  & 8867 & \cellcolor{gray!0}\textcolor{gray}{Base} \\
& Token  & 1185 & 1471 & 2993 & 9397 & 2754 & 8854 & 2084 & 3693 & 2382 & 1427 & 1146 & 3154 & \cellcolor{gray!0}\textcolor{gray}{Base} \\
[-0.4ex]\midrule\addlinespace[-0.000ex]

\rowcolor{gray!10}
& Acc.   & 75.8 & 86.3 & 30.7 & 80.0 & 17.3 & 82.7 & 30.3 & 63.1 & 30.3 & 78.2 & 35.1 & 40.1 & \gaincell{$\uparrow$ \textbf{3.4}} \\
\rowcolor{gray!10}
& Time   & 825  & 1701 & 1829 & 7784 & 1788 & 8134 & 586  & 1965 & 1194 & 1348 & 449  & 7908 & \gaincell{$\times$\textbf{1.2}} \\
\rowcolor{gray!10}
\multirow{-3}{*}{\textbf{\ours}} & Token  & 523  & 816  & 1622 & 6338 & 1576 & 7021 & 829  & 2675 & 1369 & 964  & 577  & 2198 & \gaincell{$\downarrow$ \textbf{34.6\%}} \\

\specialrule{\lightrulewidth}{0pt}{0pt}
\multicolumn{15}{>{\columncolor{gray!30}}c}{\textit{\textbf{Recursive Round r=2}}} \\
\specialrule{\lightrulewidth}{0pt}{0pt}

\multirow{3}{*}{\text{Recursive-TextMAS}} 
& Acc.   & 72.5 & 84.4 & 23.3 & 70.7 & 10.0 & 77.3 & 28.7 & 59.1 & 28.3 & 76.1 & 30.0 & 38.0 & \cellcolor{gray!0}\textcolor{gray}{Base}  \\
& Time   & 2204 & 3958 & 4247 & 14380 & 3960 & 14110 & 1825 & 4207 & 3097 & 2745 & 1847 & 14792 & \cellcolor{gray!0}\textcolor{gray}{Base}  \\
& Token  & 2117 & 2794 & 5318 & 16372 & 4982 & 16213 & 3708 & 6128 & 4436 & 2609 & 1998 & 5369 & \cellcolor{gray!0}\textcolor{gray}{Base}  \\
[-0.4ex]\midrule\addlinespace[-0.000ex]

\rowcolor{gray!10}
& Acc.   & 76.6 & 87.1 & 33.3 & 86.0 & 18.7 & 84.0 & 32.3 & 64.6 & 31.2 & 78.3 & 36.9 & 41.3 & \gaincell{$\uparrow$ \textbf{6.0}} \\
\rowcolor{gray!10}
& Time   & 1096 & 1974 & 2367 & 8178 & 2263 & 8965 & 752  & 2342 & 1427 & 1664 & 627  & 8329 & \gaincell{$\times$\textbf{1.9}} \\

\rowcolor{gray!10}
\multirow{-3}{*}{\textbf{\ours}} & Token  & 495  & 953  & 1614 & 5314 & 1552 & 6657 & 813  & 2521 & 1383 & 1008 & 531  & 2020 & \gaincell{$\downarrow$ \textbf{65.5\%}} \\

\specialrule{\lightrulewidth}{0pt}{0pt}
\multicolumn{15}{>{\columncolor{gray!30}}c}{\textit{\textbf{Recursive Round r=3}}} \\
\specialrule{\lightrulewidth}{0pt}{0pt}

\multirow{3}{*}{\text{Recursive-TextMAS}} 
& Acc.   & 69.1 & 85.8 & 18.0 & 73.3 & 16.7 & 74.7 & 28.7 & 58.6 & 28.5 & 77.1 & 29.3 & 36.5 & \cellcolor{gray!0}\textcolor{gray}{Base}  \\
& Time   & 2952 & 6010 & 6183 & 19304 & 5907 & 19678 & 3322 & 7537 & 4684 & 3922 & 2310 & 22036 & \cellcolor{gray!0}\textcolor{gray}{Base}  \\
& Token  & 3059 & 4100 & 8645 & 23651 & 7813 & 22915 & 5820 & 8091 & 6307 & 3731 & 2676 & 7078 & \cellcolor{gray!0}\textcolor{gray}{Base}  \\
[-0.4ex]\midrule\addlinespace[-0.000ex]

\rowcolor{gray!10}
& Acc.   & 77.8 & 88.2 & 34.0 & 86.7 & 20.0 & 86.0 & 32.6 & 66.2 & 31.7 & 79.3 & 37.4 & 42.8 & \gaincell{$\uparrow$ \textbf{7.2}} \\
\rowcolor{gray!10} 
& Time   & 1360 & 2320 & 2727 & 8981 & 2629 & 9623 & 861  & 2638 & 1704 & 1912 & 805  & 10186 & \gaincell{$\times$\textbf{2.4}} \\
\rowcolor{gray!10}
\multirow{-3}{*}{\textbf{\ours}} & Token  & 519  & 893  & 1586 & 5342 & 1537 & 6860 & 786  & 2524 & 1378 & 1056 & 595  & 2247 & \gaincell{$\downarrow$ \textbf{75.6\%}} \\
[-0.4ex]\bottomrule
\end{tabular}
}
\vspace{-5pt}
\end{table}

\vspace{-10pt}
\subsection{Scaling Performance via Recursion}

We begin by evaluating how \ours performs across different recursion depths $r=1,2,3$. As shown in Table \ref{tab:seq_mas}, we analyze agent collaboration behavior from three complementary perspectives: (i) accuracy, (ii) end-to-end runtime, and (iii) overall system token throughput. We also include a text-based recursive baseline for reference. 
Across seven math, science, and code generation tasks, both light and scaled versions of \ours exhibit a consistent upward trend as recursion depth increases. When compared with the text-based recursion, \ours consistently improves over the baseline by an average of 8.1\% at $r=1$, 19.6\% at $r=2$, and 20.2\% at $r=3$, with performance advantage more pronounced as the recursion deepens. Additionally, under identical MAS architectures, \ours delivers steadily increasing efficiency gains across recursion rounds, accelerating end-to-end inference time from \(1.2\times\) to \(2.4\times\) while reducing output tokens from \(34.6\%\) to \(75.6\%\). Additional case studies on the running pipeline of \ours across domains are provided in Appendix~\ref{app:case_study_question}.

\textbf{Scaling Law on \ours (Training v.s. Inference).} We further examine the scaling behavior of recursion in \ours by jointly varying the training-time and inference-time recursion rounds. Figure \ref{fig:overall} (Up) illustrates the performance landscape of \ours under different training and inference settings. 
Increasing inference depth continues to improve systems trained with fewer rounds, while deeper training shifts the entire performance frontier upward, with the strongest results consistently appearing in the upper-right region where both are large. 
This trend suggests a complementary training-inference scaling effect in \ours: training recursion progressively teaches the system to form refinement-ready latent states, and subsequent inference recursion translates this learned recursive structure into additional test-time gains.

\begin{table}[!t]
\centering \small
\caption{\textbf{Comparison of \ours with Other Methods.} 
We evaluate \ours at recursion round $r=3$. Under the same training budget and model setups, \ours consistently outperforms advanced single-agent methods, alternative MAS frameworks, and recursive computation baselines.
}
\vspace{-5pt}
\label{tab:trainable_comparison}
\resizebox{\linewidth}{!}{
\begin{tabular}{lcccccc}
\toprule
\textbf{Method} & \textbf{MATH500} & \textbf{AIME2025} & \textbf{AIME2026} & \textbf{GPQA-D} & \textbf{LiveCodeBench} & \textbf{MedQA} \\ \midrule
Single Agent (w/ LoRA) & 83.1 & 70.0 & 73.3 & 62.0 & 37.4 & 76.1 \\
Single Agent (w/ Full-SFT) & 83.2 & 73.3 & 76.7 & 62.8 & 38.6 & 77.0 \\
\midrule
Mixture-of-Agents (MoA) & 79.8 & 60.0 & 63.3 & 47.6 & 27.0 & 57.5 \\
TextGrad & 84.9 & 73.3 & 76.7 & 62.5 & 39.8 & 77.2 \\
\midrule
LoopLM & 84.6 & 66.7 & 63.3 & 48.1 & 24.9 & 56.4 \\
Recursive-TextMAS & 85.8 & 73.3 & 73.3 & 61.6 & 38.7 & 77.0 \\
\midrule
\textbf{\ours} & \textbf{88.0} & \textbf{86.7} & \textbf{86.7} & \textbf{66.2} & \textbf{42.9} & \textbf{79.3} \\
\bottomrule
\end{tabular}
}
\vspace{-10pt}
\end{table}

\vspace{-10pt}
\subsection{Broader Comparison with Alternative Architectures and Training Frameworks}  

Table \ref{tab:trainable_comparison} compares \ours at the whole-system level against a broader set of baselines, including single fine-tuned agents, representative multi-agent frameworks, and alternative recursive methods. 
To ensure fair comparison, all methods are instantiated with identical backbone models and comparable training budgets (e.g., matched trainable parameter counts, recursion depth, training set).

Overall, \ours delivers a consistent whole-system advantage, achieving an average performance improvement of 8.3\% over the strongest baseline on each benchmark. With the same training data, fine-tuning individual agents strengthens performance relative to their off-the-shelf versions, while \ours delivers further gains by optimizing cross-agent collaboration at the system level.
In addition, \ours remains the performance advantage compared to advanced architectures such as TextGrad and LoopLM, especially on reasoning-intensive tasks (e.g., accuracy gains of 18.1\% on AIME2025, 13.0\% on AIME2026, and 5.4\% on GPQA-Diamond).

\vspace{-10pt}
\subsection{Can \ours Generalize across Diverse Collaboration Patterns?}

Beyond the sequential setting, we further instantiate \ours under three additional MAS collaboration patterns in Table \ref{tab:mas_config} to assess whether our method is agnostic to any specific system architecture and generalizes across diverse usage scenarios. As shown in Figure~\ref{fig:overall} (Down), we compare the accuracy of \ours against strong standalone agents within each collaboration pattern.

In \textit{Mixture-style}, \ours achieves an average improvement of 6.2\% over the strongest domain specialist on each benchmark, suggesting that recursive interaction enables non-trivial cross-domain composition beyond what can be attained by selecting one individual specialist alone. 
In \textit{Deliberation-style}, we evaluate tool use on both mathematical and search-intensive tasks. \ours improves the original tool-calling agent by 4.8\%, showing that recursive latent coordination remains effective in tool-calling settings through iterative interaction with the Reflector. 
Finally, in \textit{Distillation-style}, \ours improves the learner by 8.0\% while retaining 1.5$\times$ end-to-end speed advantage over the expert. In this way, \ours distills much of the expert’s capability into a more efficient system.
We leave detailed reports of Figure~\ref{fig:overall} (Down) in Appendix~\ref{app:collab_style}.

\begin{figure}[!t]
    \centering
    \includegraphics[width=\linewidth]{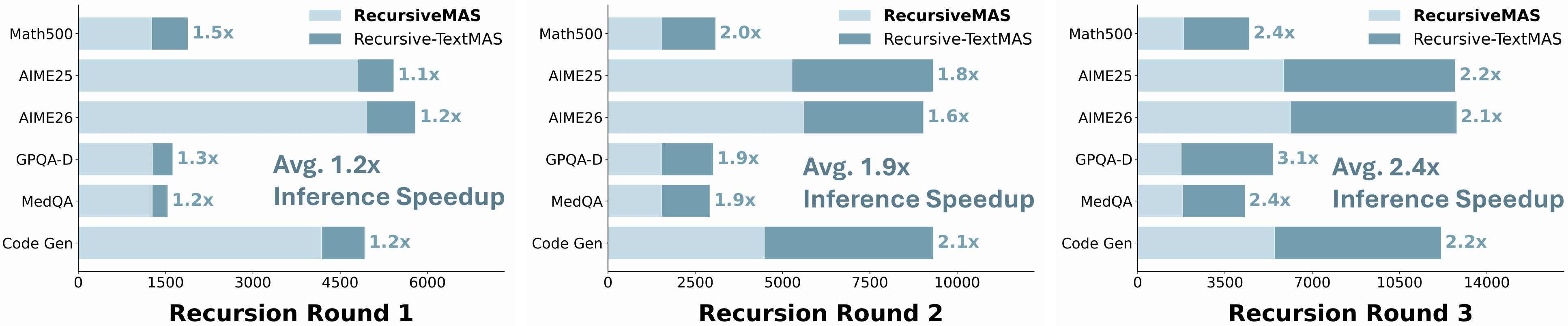}
    \caption{\textbf{Inference Time Speedup of \ours across Three Recursion Rounds.} 
    \ours exhibits increasing inference speedup as the recursion depth increases.
    }
    \label{fig:time_speedup}
\end{figure}

\begin{figure}[!t]
    \centering
    \includegraphics[width=\linewidth]{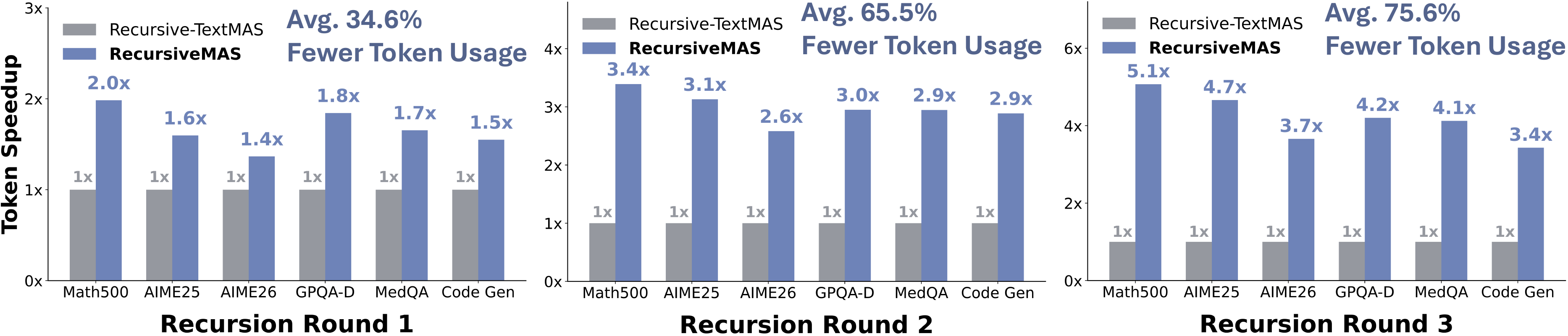}
    \caption{
    \textbf{Token Reduction of \ours across Three Recursion Rounds.} 
    As recursion deepens, \ours reduces substantially more tokens than Recursive-TextMAS.
    }
    \label{fig:token_speedup}
    \vspace{-10pt}
\end{figure}

\vspace{-6pt}
\subsection{Efficiency Analyses on Latent-space Recursion}
\textbf{Inference Time Speedup.} We analyze the efficiency of \ours 
to empirically support our complexity advantage in Proposition~\ref{prp:eff}. We first compare \ours against Recursive-TextMAS to study how our advantage on end-to-end inference time scales with recursion depth. As shown in Figure~\ref{fig:time_speedup}, although deeper recursion rounds introduce cost, we find that \ours consistently exhibits efficiency gain, and the advantage further increases as recursion deepens. 
For example, at recursion round $r=1$, \ours already achieves a $1.2\times$ speedup on average, and this advantage grows to $1.9\times$ and $2.4\times$ at larger recursion rounds of $r=2/3$. 
This trend aligns well with our method design, where \ours achieves a favorable scaling behavior by conducting recursive collaboration directly in latent space and avoiding repeated intermediate text generation.

\textbf{Overall Token Usage Reduction.} We next demonstrate the substantial token usage reduction of \ours in Figure~\ref{fig:token_speedup}. 
Within the comparison, we find that the baseline method suffers from rapidly growing token overhead as recursion round increases, while \ours reduces the token usage by $34.6\%$ for the first recursion round, and the reduction scales to $75.6\%$ at $r=3$. 
This is because Recursive-TextMAS repeatedly decode the intermediate text at every recursion round, whereas \ours performs most recursive interaction directly in latent space. 
Overall, \ours enables a much more efficient system-level scaling behavior, and the resulting efficiency gain is amplified as the number of recursion rounds increases.

\vspace{-6pt}
\section{In-depth Analyses on RecursiveMAS}

\textbf{\link Design.}
To validate the effectiveness of \link, we compare our 2-layer residual design against three alternatives: (i) a 1-layer network, (ii) a 1-layer network with the residual connection, and (iii) a 2-layer network without the residual connection. We conduct experiments using the scaled sequential-style \ours and adapt the same architecture for both $\mathcal R_{\text{in}}$ and $\mathcal R_{\text{out}}$.

\begin{wraptable}{r}{0.55\textwidth}
\vspace{-10pt}
\centering
\small
\caption{\textbf{Efficacy on \link Design.} We compare accuracy across alternative architectural designs.}
\label{tab:ablation_link_design}
\vspace{-5pt}
\resizebox{\linewidth}{!}{
\begin{tabular}{lccc}
\toprule
\textbf{\link Design} & Math500 & GPQA-D & LiveCodeBench \\ \midrule
1-Layer      & 84.4 & 63.2 & 40.1 \\
\midrule
Res+1-Layer  & 86.7 & 65.3 & 41.4 \\
\midrule
2-Layer      & 85.6 & 64.5 & 40.5 \\
\midrule
\textbf{Res+2-Layer (ours)}  & \textbf{88.0} & \textbf{66.2} & \textbf{42.9} \\
\bottomrule
\end{tabular}
}
\vspace{-10pt}
\end{wraptable}

As shown in Table~\ref{tab:ablation_link_design}, our 2-layer residual design performs best across all three benchmarks, and the residual connection delivers additional improvements across different backbone models.
For example, on GPQA-Diamond, equipping a single-layer design with a residual branch improves the performance from 63.2\% to 65.3\%, which is even higher than the plain 2-layer design (64.5\%).
These results align with our design intuition in Section~\ref{sec:recursive_link}: by preserving latent semantics while learning only the distributional shift, \link achieves stable training and stronger inference performance.

\textbf{Semantic Representations in Recursion.}
We analyze how the semantic distribution of \ours changes across different recursion rounds. 
Under the scaled sequential setting of \ours, we randomly sample 500 question-answer pairs spanning all downstream domains. We then use the solver agent’s input embedding layer to map each ground-truth answer string into embedding representations, which serves as the reference semantic distribution.
We run \ours at recursion rounds $r=1,2,3$ to generate final answers for all these 500 questions, map the generated answers into embeddings using the same input embedding layer, and visualize both the ground-truth reference (\textcolor[HTML]{797e9b}{"purple"}) and newly generated distributions (\textcolor[HTML]{FC936B}{"orange"}) via PCA projection.

\begin{figure}[!h]
    \centering
    \includegraphics[width=\linewidth]{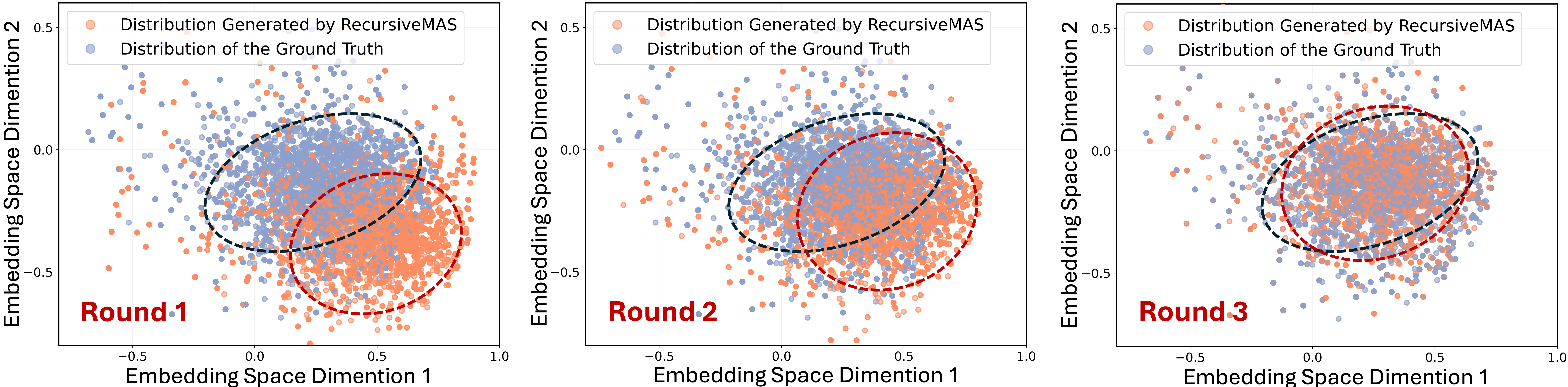}
    \caption{
    \textbf{Semantic Representations of \ours across Differnt Recursion Rounds.} 
    We visualize the semantic distribution of the final answers generated by \ours and the corresponding ground-truth across 500 questions. 
    Increasing recursion rounds progressively aligns the generated distribution of \ours with the ground truth distribution.
    }
    \label{fig:analyses_embedding}
\end{figure}

In Figure~\ref{fig:analyses_embedding}, the generated answers at \(r=1\) remain visibly shifted from the ground-truth distribution, but this discrepancy progressively narrows as depth increases, with the two distributions becoming largely aligned by \(r=3\).
This aligning trend suggests that \ours iteratively refines the latent embeddings and corresponding answers through recursion.
We further take a closer look to examine individual test instances and provide detailed case studies in Appendix~\ref{app:case_diff_recursion}.
Our case studies reveal a common pattern in which \ours may produce an incorrect answer at an early stage, while deeper recursion successfully corrects it through iterative refinement.
Together, these analyses provide further evidence that latent thoughts capture semantically meaningful representations, and that deeper recursion improves alignment toward correct final outputs.

\begin{wrapfigure}{r}{0.43\textwidth}
    \centering
    \includegraphics[width=\linewidth]{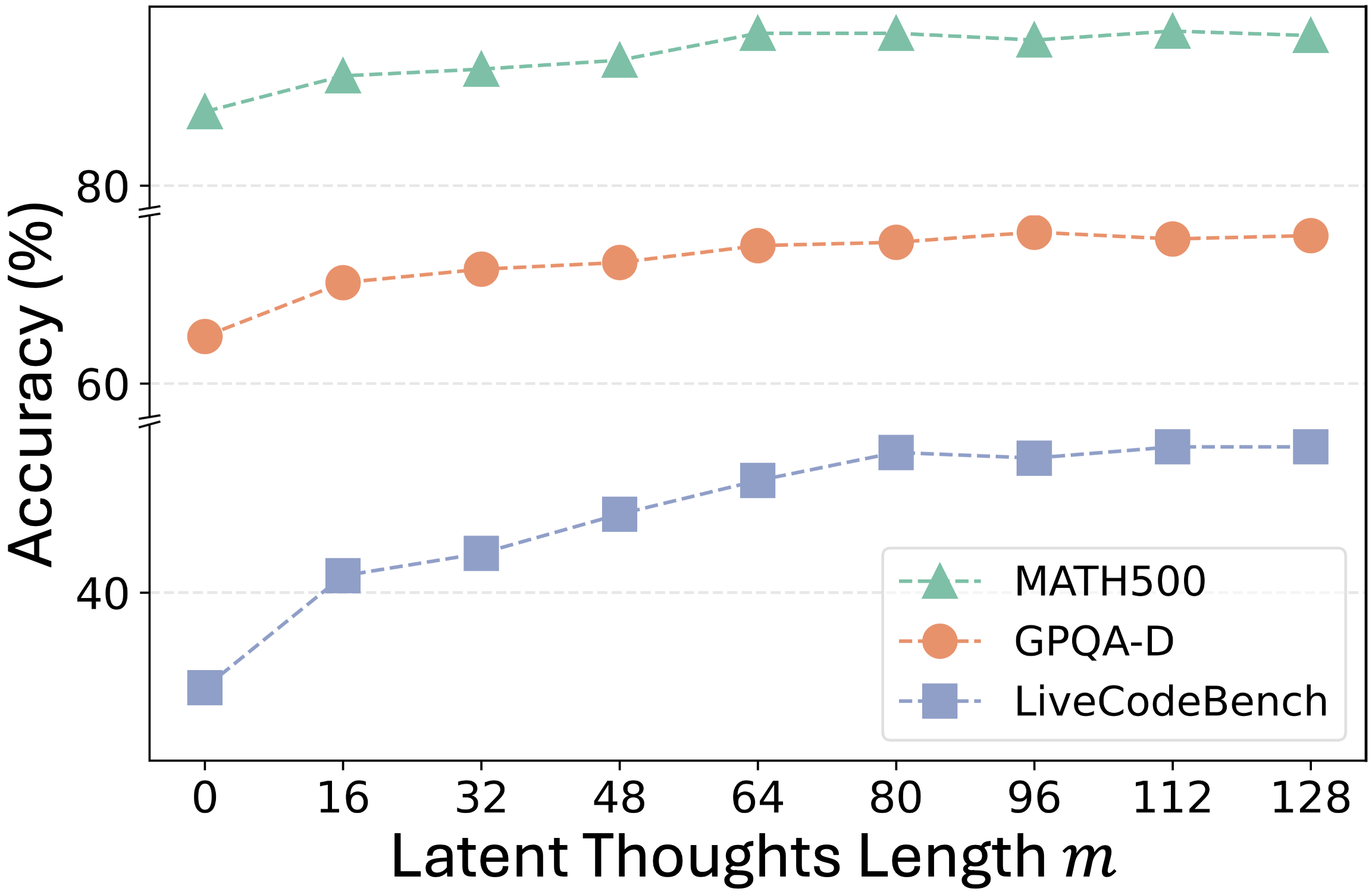}
    \vspace{-15pt}
    \caption{Effectiveness of \ours's latent thoughts with different step lengths.}
    \label{fig:latent_steps}
    \vspace{-10pt}
\end{wrapfigure}
\textbf{Optimal Length of Latent Thoughts Generation.}
We next study and ablate the latent thoughts length $m$ to examine how much of each agent's internal reasoning is sufficient to support effective collaboration. 
Under the scaled sequential-style of \ours, we evaluate a broad range of $m$.
As illustrated in Figure~\ref{fig:latent_steps}, increasing \(m\) improves performance in the early regime. Once \(m\) reaches a moderate scale (around \(m=80\)), performance is stabilized across all benchmarks.
The ablation suggests that \ours enables effective agent reasoning and interaction with only a modest latent-thought budget, in sharp contrast to text-based collaboration that typically requires longer CoT and costly token generation.

\textbf{Training Cost Analysis}
We further analyze the training cost of \ours under the scaled sequential-style MAS setting. 
We compare \ours with direct training methods, including LoRA and full supervised fine-tuning with the same training data and backbone setup. For cost estimation, we follow prior methods \citep{lu2023empirical, liu2025deepseek} to measure the cost based on GPU usage.
As shown in Table~\ref{tab:cost_analysis}, \ours{} utilizes the lowest per-agent GPU memory, trainable parameter count, and estimated cost among all compared training strategies. 
Meanwhile, \ours{} achieves the highest accuracy across all downstream tasks, suggesting that optimizing the lightweight \link{} provides a better cost-performance trade-off than other training methods.

\begin{table}[!t]
\centering
\small
\caption{\textbf{Cost analysis on \ours.}
We report the peak GPU memory usage (GB), number of trainable parameters, estimated cost, and average accuracy (\%) across all downstream tasks.}
\label{tab:cost_analysis}
\resizebox{0.75\linewidth}{!}{
\begin{tabular}{lcccc}
\toprule
\textbf{Methods} & \textbf{GPU Mem.} & \textbf{Trainable Param.} & \textbf{Cost} & \textbf{Avg. Acc.} \\
\midrule
LoRA Training & 21.67 & 15.92M (0.37\%) & \$6.64 & 66.9 \\
\midrule
Full-SFT & 41.40 & 4.21B (100\%) & \$9.67 & 68.6 \\
\midrule
\textbf{\ours} & \textbf{15.29} & \textbf{13.12M (0.31\%)} & \textbf{\$4.27} & \textbf{74.9} \\
\bottomrule
\end{tabular}
}
\vspace{-10pt}
\end{table}

\vspace{-6pt}
\section{Related Works}

\textbf{LLM-based Multi-Agent Systems.}
Current LLMs achieve strong performance on general tasks, but they often exhibit bottlenecks when facing diverse reasoning patterns \citep{mirzadeh2025gsm,valmeekam2023planbench,maheswaran2026squeeze} or domain-specific challenges \citep{chen2025benchmarking}.
To overcome these limitations, Multi-agent systems extend the single LLM paradigm to a collaborative setting \citep{wu2024autogen, yang2024llm, tran2025multi, su2025toolorchestra} by organizing a set of agents with distinct roles that jointly address the problem.
A standard multi-agent system topology involves a sequential configuration \citep{li2023camel, qian2024chatdev}, where agents are assigned in a linear pipeline to decompose and resolve problems in order. Beyond sequential settings, other works also explore mixture-style settings \citep{yun2026graph, ye2025x, wang2025mixture}, where multiple agents with domain expertise reason in parallel, and their outputs are aggregated into a final decision. 
Another line of work seeks to improve MAS through textual feedback signals. For example, related optimization methods 
\citep{yuksekgonul2025optimizing, shen2025optimizing} leverage an LLM to generate natural language feedback to refine contextual inputs and instructions of each agent. Additionally, another study \citep{motwani2024malt} improves MAS by separately training each agent with role-specific responses.
Rather than separate training each individual agent or only leveraging textual feedback, \ours treats MAS as a unified whole, and scales the system performance via recursively refining the latent information flow.

\textbf{Scaling Reasoning via Recursion.} Recent studies explore recursion as an alternative scaling axis for LLMs \citep{geiping2025scaling, li2026learning, bae2025mixture, tang2026looprpt}, where the same computation blocks are reused through multiple recurrent rounds (i.e., loops) to increase reasoning depth and iteratively refine hidden representations. One line of work studies recursive language models that apply shared layers to scale latent reasoning. For instance, LoopLM \citep{zhu2025scaling} introduces pre-trained looped language models with iterative latent computation.
Besides, other work explores other recursive architectures \citep{jolicoeur2025less, zhang2025recursive, wang2025hierarchical}, including tiny recursive networks and recursive self-calling schemes for long-context inference. 
While existing methods in agentic AI primarily focus on recursion inside a single language model, \ours exhibits the first attempt to extend the recursive scaling paradigm to system-level. Additional related works are provided in Appendix~\ref{app:addition_related}.

\vspace{-5pt}
\section{Conclusion}
We introduce \ours, a recursive multi-agent framework that scales agent collaboration through system-level recursion. 
\ours first supports latent-thoughts generation within each agent through inner \link, then connects heterogeneous agents through outer \link, and optimizes the whole system with an inner-outer loop training paradigm.
Theoretically, our framework leads to more stable training dynamics and improves efficiency compared to text-based baselines. 
Our empirical results across mathematical and scientific reasoning, code generation, and search benchmarks show that \ours consistently improves accuracy while substantially reducing inference time and token usage.
Overall, \ours provides a scalable and efficient framework for multi-agent systems to recursively collaborate, refine, and evolve in latent space.

\bibliographystyle{abbrvnat}
\nobibliography*

\begingroup
\hypersetup{urlcolor=black}
\bibliography{main}
\endgroup

\newpage 

\appendix
\renewcommand{\contentsname}{\Large Table of Contents}
{\hypersetup{linkcolor=black}
\tableofcontents
}

\addtocontents{toc}{\protect\setcounter{tocdepth}{2}}

\newpage
\textbf{\Large Appendix}

\section{Theoretical Analysis}

\subsection{Running Complexity Analysis}
\label{app:proposition_proof}

\begin{proposition}[\textbf{Restate of Proposition~\ref{prp:eff}}]
Without \link, a text-based Recursive MAS with the same collaboration structure requires runtime complexity of $\Theta(N(m|V|d_h+(t+m)d_h^2+(t+m)^2d_h))$; In contrast, with \link-enabled collaboration, \ours{} achieves an end-to-end runtime complexity of $\Theta(N(md_h^2+(t+m)d_h^2+(t+m)^2d_h))$. 
\end{proposition}

\begin{proof}[Proof of Proposition~\ref{prp:eff}]
We analyze the runtime complexity for each agent and then extend the result to the full multi-agent system. For each single agent, given the context length is at most $t$, and the generation length is at most $m$, the Transformer processes a sequence of length at most $t+m$, requiring $\Theta((t+m)d_h^2)$ time for the feed-forward layers and $\Theta((t+m)^2d_h)$ time for self-attention. This standard Transformer computation is shared by both \ours{} and text-based Recursive MAS.

The computational difference comes from how \ours process the generated embeddings.
In \ours{}, each of the $m$ latent embeddings is transformed by \link{}, which incurs an additional cost of $\Theta(md_h^2)$.
In text-based Recursive MAS, each embedding must be decoded into an explicit token by projecting it to the vocabulary space and computing logits over $|V|$ tokens, resulting in an additional cost of $\Theta(m|V|d_h)$.

Adding all terms together, our proposed \ours{} needs $\Theta(md_h^2+(t+m)d_h^2+(t+m)^2d_h)$ time for each agent, while text-based Recursive MAS needs $\Theta(m|V|d_h+(t+m)d_h^2+(t+m)^2d_h)$ time.
Since there are $N$ agents in the system, our proposed \ours{} needs $\Theta(N(md_h^2+(t+m)d_h^2+(t+m)^2d_h))$ time, and text-based Recursive MAS needs $\Theta(N(m|V|d_h+(t+m)d_h^2+(t+m)^2d_h))$ time in total. 
\end{proof}

\subsection{Realistic Assumptions}
\label{app:assumptions}
\begin{ASS}
Text-based SFT can be regarded as using $\mathcal R_\text{text}(h)=W_\text{in}\operatorname{softmax}(W_\text{out}h)$ as the recursive link, where $W_\text{in}$ is the token-to-embedding matrix, and $W_\text{out}$ denotes the embedding-to-logits matrix. 
We also assume $\|W_\text{in}\|_2\le O(1)$ and $\|W_\text{out}\|_2\le O(1)$. For RecursiveLink $\mathcal R(h)$, we assume that $W_1,W_2$ follow Kaiming normal initialization, and we only analyze the case where $W_3=I$.
\end{ASS}

\subsection{Learning Advantage Analysis}
\label{app:proof-grad}
\label{app:theorem_proof}

\begin{theorem}[\textbf{Restate of Theorem~\ref{thm:grad}}]
Under the Realistic Assumptions (stated in Appendix \ref{app:assumptions}), if tokens are confident with entropy $\le \epsilon$, where typically $\epsilon \ll 1$: directly applying text-based SFT (denoted by $\mathcal R_\text{text}(h)$) during recursion suffers from gradient vanishing (i.e., gradient norm close to 0); while \ours with the \link $\mathcal{R}$ maintains stable and near constant gradients (i.e., gradient norm close to 1) during looped backpropagation process. Formally, with probability $\ge 1-\delta$,

\begin{align}
\left\|\frac{\partial\mathcal R_\text{text}(h)}{\partial h}\right\|_2\le O(\epsilon)\ll1,\qquad\left\|\frac{\partial\mathcal R(h)}{\partial h}\right\|_2\ge\Omega\left(1-\sqrt{\frac1{d_h}\log\frac1{\delta}}\right).%
\end{align}
\end{theorem}

\begin{proof}[Proof of Theorem~\ref{thm:grad}]
We first analyze the gradient behavior of text-based recursive interaction.
By applying the chain rule to $\mathcal{R}_\text{text}(h)$, the gradient matrix is:
$$J_\text{text}=\frac{\partial\mathcal R_\text{text}(h)}{\partial h} = W_\text{in} S W_\text{out},\qquad S = \text{diag}(p) - p p^T,$$
where $p=\operatorname{softmax}(W_\text{out}h)$ is the next token distribution. Then, by the sub-multiplicativity of the spectral norm,
$$\|J_\text{text}\|_2 \le \|W_\text{in}\|_2 \|S\|_2 \|W_\text{out}\|_2 \le O(1) \cdot \|S\|_2 \cdot O(1) = O(\|S\|_2).$$
Because $S$ represents the covariance matrix of a categorical distribution, it is symmetric and positive semi-definite. Thus, its spectral norm is upper-bounded by its trace:
$$\|S\|_2 \le \text{Tr}(S) = \sum_{i=1}^{|V|} (p_i - p_i^2) = \sum_{i=1}^{|V|} p_i - \sum_{i=1}^{|V|} p_i^2 = 1 - \|p\|_2^2.$$
Using the logarithm inequality $\ln z \le z - 1$ (for all $z>0$), we can lower-bound the entropy:
$$\epsilon \ge \sum_{i=1}^{|V|}p_i(-\ln p_i)\ge \sum_{i=1}^{|V|} p_i(1 - p_i) = 1 - \|p\|_2^2.$$
Substituting this inequality back into the norm bound yields:
$$\|S\|_2 \le 1 - \|p\|_2^2 \le \epsilon.$$
Therefore, combining the constants, we have:
$$\left\|\frac{\partial\mathcal R_\text{text}(h)}{\partial h}\right\|_2=\|J_\text{text}\|_2\le O(\epsilon).$$

We next analyze the gradient behavior of \ours.
Applying the chain rule to $\mathcal{R}(h)$, the gradient matrix is:
$$J = \frac{\partial\mathcal R(h)}{\partial h} = I + W_2 \Sigma' W_1,$$
where $\Sigma' = \text{diag}(\sigma'(W_1 h))$. By the triangle inequality for the matrix norm,
$$\Big| \|J\|_2 - 1 \Big|=\Big| \|J\|_2 - \|I\|_2 \Big| \le \|J-I\|_2 = \|W_2 \Sigma' W_1\|_2.$$
Since $W_1,W_2$ follow Kaiming initialization, and the GELU function $\sigma$ has $|\sigma'|\le O(1)$, then by the subgaussian matrix concentration inequality, $\|W_2 \Sigma' W_1\|_2\le O\left(\sqrt{\frac{1}{d_h}\log\frac1{\delta}}+1\right)$ with probability $\ge1-\delta$. It follows that:
$$\|J\|_2\ge\Omega\left(1-\sqrt{\frac1{d_h}\log\frac1{\delta}}\right).$$
\end{proof}

\newpage
\section{Experiment Setups}

\subsection{Evaluation Datasets}
\label{app:datasets}
We introduce all datasets used in our experiments as follows:

\textbf{Mathematical Reasoning.}
\begin{itemize}[leftmargin=*]
    \item \textbf{MATH500} \citep{math500} is a widely used subset of the MATH benchmark, containing mathematical problems across algebra, geometry, number theory, probability, and combinatorics. 
    
    \item \textbf{AIME2025} \citep{aime25} contains 30 challenging problems from the 2025 American Invitational Mathematics Examination. These problems require olympiad-style derivations and precise numerical answers, providing a compact but difficult benchmark for mathematical reasoning.
    
    \item \textbf{AIME2026} \citep{aime26} follows the same AIME-style questions with 30 challenging competition-level math problems. We use it to further test performance and generalization on difficult mathematical reasoning tasks.
\end{itemize}

\textbf{Scientific and Medical Tasks.}
\begin{itemize}[leftmargin=*]
    \item \textbf{GPQA-Diamond} \citep{gpqa} is the most difficult split of GPQA, consisting of graduate-level multiple-choice questions in biology, physics, and chemistry. It requires specialized scientific knowledge and careful multi-step reasoning beyond shallow factual recall.
    
    \item \textbf{MedQA} \citep{medqa} contains medical licensing-style questions that assess biomedical knowledge, clinical reasoning, and diagnostic decision-making. The benchmark requires integrating domain-specific evidence from patient scenarios and selecting the most appropriate answer.
\end{itemize}

\textbf{Code Generation.}
\begin{itemize}[leftmargin=*]
    \item \textbf{LiveCodeBench-v6} \citep{lcb} is a contamination-resistant code generation benchmark built from recent programming problems. It evaluates whether models can synthesize functionally correct programs under realistic problem specifications and hidden test cases.
    
    \item \textbf{MBPP Plus} \citep{codeplus} extends the original MBPP benchmark with more comprehensive test cases for Python program synthesis. The stricter execution-based evaluation provides a more reliable measure of functional correctness.
\end{itemize}

\textbf{Search-based Question Answering.}
\begin{itemize}[leftmargin=*]
    \item \textbf{HotpotQA} \citep{yang2018hotpotqa} is a multi-hop question answering benchmark based on Wikipedia. It requires models to gather and combine evidence from multiple supporting facts, making it suitable for evaluating search-based reasoning.
    
    \item \textbf{Bamboogle} \citep{Bamboogle} is a compact but challenging benchmark for search-intensive multi-hop reasoning. Its questions often require decomposition and intermediate retrieval steps before composing the final answer.
\end{itemize}

\subsection{Compared Baselines}
\label{app:baseline}
We compare our method against the following baselines:

\textbf{Single-Agent Fine-tuning Baselines.}
\begin{itemize}[leftmargin=*]
    \item \textbf{Single Agent (w/ LoRA)} uses the final agent from the corresponding collaboration pattern and trains it with LoRA adapters using the same training data as \ours.
    
    \item \textbf{Single Agent (w/ Full-SFT)} further fine-tunes all parameters of the same single-agent backbone using fully supervised fine-tuning. 
\end{itemize}

\textbf{Representative Multi-Agent Frameworks.}
\begin{itemize}[leftmargin=*]
    \item \textbf{Mixture-of-Agents (MoA)}~\citep{wang2025mixture} organizes multiple LLM agents into a layered multi-agent system, where agents in each layer aggregate responses from the previous layer to produce refined outputs as the final answer. 
    
    \item \textbf{TextGrad}~\citep{yuksekgonul2025optimizing} optimizes multi-agent systems by propagating natural-language feedback as textual gradients. We use TextGrad as a baseline method for text-mediated system optimization.
\end{itemize}

\textbf{Recursion-based Methods.}
\begin{itemize}[leftmargin=*]
    \item \textbf{LoopLM}~\citep{zhu2025scaling} is a looped language model that performs recursive computation with shared transformations in latent space. In our evaluation, we mainly adopt the Ouro-2.6B model.
    
    \item \textbf{Recursive-TextMAS} uses the same recursive multi-agent collaboration structure as \ours, but agents communicate through explicit text rather than latent representations. 
\end{itemize}

\subsection{Additional Implementation Details}
\label{app:train}

\textbf{Training Data Curation.}
To optimize \ours under our inner-outer training pipeline, we construct role-specific supervision targets for each agent across all collaboration patterns.
We start by collecting question-answer pairs as raw training samples from four domains, including s1K~\citep{s1}, m1K~\citep{m1}, OpenCodeReasoning~\citep{opencode}, and ARPO-SFT~\citep{dong2025agentic}. 
For each training sample, we rewrite the original answer into agent-level targets according to the role assignments of each collaboration pattern, as follows:

\begin{itemize}[leftmargin=*]
    \item For \textit{Sequential-Style}, we use Qwen3.5-397B-A17B to rewrite the answers into an initial step-by-step plan and a refined critic-guided plan. During training, the initial plan is used as the supervision target for the Planner, the critic-guided plan is used for the Critic, and the original answer is retained for the Solver.
    \item For \textit{Mixture-Style}, each specialist in the MAS first generates responses for questions from its specialized domain, and these responses are then used to supervise the corresponding specialist. The ground truth answers are used as targets for the Summarizer.
    \item For \textit{Distillation-Style}, the Expert first generates guidance-style responses for each training sample, which are then used as supervision targets for the Expert. The Learner is supervised directly by the ground-truth answers.
    \item For \textit{Deliberation-Style}, we use the ground truth answers as the supervision targets for both the Reflector and Tool-Caller agent.
\end{itemize}

After the role-specific data construction, each agent is assigned its own training pairs, where the input consists of the original question, and the output is the corresponding role-based response.
These agent-level pairs are then used as the supervision data for subsequent training.

\textbf{Implementation Details.}
During training, all base LLMs' parameters are frozen, and we only update the inner and outer \link using AdamW with a batch size of 4 and a maximum sequence length of 4096 tokens.
During inference, the maximum generation length is set for 2000 tokens for MATH500, 4000 tokens for MedQA, GPQA-Diamond, LiveCodeBench, and MBPP Plus, and 16000 tokens for AIME2025/2026. 
For all Qwen models, we enables the official Instruct mode \citep{qwen3.5} for more efficient and controllable answer generation. 
For the Deliberation-Style MAS, we provide a standard Python environment and a Tavily~\citep{tavily} search API as external tools.
We implement \ours and baselines with both HuggingFace Transformer \citep{huggingface_transformers_docs} and vLLM backend \citep{kwon2023efficient}. All experiments are conducted on H100 and A100 GPUs.

\textbf{Evaluation Protocol.} 
Across all non-code generation tasks, we first normalize the extracted answer by removing extra whitespace, stripping punctuation, and converting text to lowercase before applying task-specific correctness checks.
\begin{itemize}[leftmargin=*]
    \item For \textit{numerical problems (MATH500, AIME2025, AIME2026)}, we compare the numerical value of the extracted answer with the ground truth. The answer is considered correct if the two values are mathematically equivalent. 

    \item For \textit{multiple-choice questions (GPQA-Diamond MedQA)}, we directly compare the predicted choice with the ground truth letter, where an exact option match is counted as correct. 

    \item For \textit{code generation tasks (LiveCodeBench and MBPP Plus)}, we first extract the generated code block and then execute it with provided test cases in a sandboxed python environment. Each individual test case is assigned a timeout of 10 seconds to prevent non-terminating programs. 

    \item For \textit{search-based tasks (HotpotQA, Bamboogle)}, we follow the standard LLM-as-a-judge evaluation method~\citep{li2025flow} and use the Qwen3.5-397B-A17B model as a binary judge to determine whether the generated answer is correct with respect to the ground truth answer. 
\end{itemize}

When an output reaches the maximum generation length without producing an extractable answer, we follow standard early-stopping evaluation methods~\citep{s1, yang2025qwen3} by appending ``Final Answer:'' to the model output to elicit a final response.

\newpage
\section{Additional Related Work}
\label{app:addition_related}
\textbf{Latent-Space Collaboration.}
Beyond text-based interaction, recent studies have explored leveraging the latent space as an alternative medium for LLM communication. 
One line of work studies transferring hidden embeddings for cross-model communication~\citep{yu2026latent, du2025enabling}, while other works investigate reusing internal states to share information across LLMs~\citep{fu2025cache, ye2025kvcomm}.
Recent studies extend this scheme to agentic settings, where latent interfaces are used to support coordination among multiple agents~\citep{zheng2025thought, zou2025latentmas}. 
Different from these studies, \ours treats latent information as part of a system-level recursive information flow, enabling heterogeneous agents to recursively collaborate and improve as a unified MAS.

\begin{table}[!t]
\centering \small
\caption{Comparison of \ours in Distillation-Style Multi-agent System. \ours improves the Learner agent by 8.0\% via distilling knowledge from the Expert agent while retaining a 1.5× end-to-end speed advantage over the Expert agent.}

\label{tab:distill_mas}
\begin{tabular}{l c ccccc}
\toprule
\textbf{Method (Distillation-Style)} & \textbf{Metric} & \textbf{AIME2026} & \textbf{GPQA-D} & \textbf{LiveCodeBench} & \textbf{MBPP+} & \textbf{MedQA} \\ \midrule
\multirow{2}{*}{Expert Model} & Acc. & 90.0 & 72.7 & 46.2 & 73.4 & 86.0 \\
& Time & 9473 & 2558 & 9352 & 2342 & 2124 \\
\midrule
\multirow{2}{*}{Learner Model} & Acc. & 76.7 & 61.4 & 38.4 & 67.5 & 77.9  \\
& Time & 4495 & 1289 & 5396 & 1171 & 1183 \\
\midrule
\multirow{2}{*}{\textbf{\ours}} & Acc.  & 83.3 & 70.0 & 40.1 & 71.9 & 83.0 \\
& Time & 5967 & 1671 & 6863 & 1516 & 1436 \\
\bottomrule
\end{tabular}
\end{table}

\section{Additional Experiments}

\subsection{Results on Different Collaboration Patterns}
\label{app:collab_style}

\begin{table}[h]
\centering \small
\caption{Comparison of \ours in Mixture-Style Multi-agent System.}
\label{tab:mixture_mas}
\resizebox{0.85\linewidth}{!}{
\begin{tabular}{l cccc}
\toprule
\textbf{Method (Mixture-Style)} & \textbf{AIME2026} & \textbf{GPQA-Diamond} & \textbf{LiveCodeBench} & \textbf{MedQA} \\ \midrule
Math Specialist & 43.3 & 37.4 & 18.9 & 29.0 \\
Code Specialist & 13.3 & 26.2 & 21.5 & 43.3 \\
Science Specialist & 10.0 & 27.0 & 7.6 & 48.1 \\
\midrule
\textbf{\ours} & \textbf{46.7} & \textbf{43.0} & \textbf{23.8} & \textbf{61.7} \\
\bottomrule
\end{tabular}
}
\end{table}

\begin{table}[h]
\centering \small
\caption{Comparison of \ours in Deliberation-Style Multi-agent System.}
\label{tab:deliberation_mas}
\resizebox{0.85\linewidth}{!}{
\begin{tabular}{l cccc}
\toprule
\textbf{Method (Deliberation-Style)} & \textbf{AIME2026} & \textbf{GPQA-Diamond} & \textbf{HotpotQA} & \textbf{Bamboogle} \\ \midrule
Reflector & 76.7 & 61.2 & 27.5 & 40.9 \\
Tool-Caller & 86.7 & 63.1 & 39.6 & 49.8 \\
\midrule
\textbf{\ours} & 90.0 & 65.0 & 41.4 & 53.7 \\
\bottomrule
\end{tabular}
}
\end{table}

We report the detailed results of \ours under three additional collaboration patterns in Tables~\ref{tab:mixture_mas}, \ref{tab:distill_mas}, and \ref{tab:deliberation_mas}, corresponding to the summarized results in Figure~\ref{fig:overall} (Down). 
In both Mixture and Deliberation settings, \ours achieves consistent accuracy gains over the strongest individual agent in each setting. 
In Distillation Style, \ours improves performance over the Learner while requiring substantially less inference time than the Expert. 
Overall, these results show that \ours provides both performance gains and efficiency benefits across diverse MAS collaboration patterns, further demonstrating the generality of our method.

\subsection{Ablations on Latent Thoughts Lengths}
\label{app:latent_steps}
\begin{table}[h]
\centering \small
\caption{Ablation Study on Length of Latent Thoughts $m$ transferred across agents on \ours.}
\label{tab:ablation_latent_steps}
\resizebox{0.8\linewidth}{!}{
\begin{tabular}{l ccccccccc}
\toprule
\textbf{Latent Steps} & \textbf{0} & \textbf{16} & \textbf{32} & \textbf{48} & \textbf{64} & \textbf{80} & \textbf{96} & \textbf{112} & \textbf{128} \\ \midrule
Math500 & 83.3 & 84.9 & 85.2 & 85.6 & 86.8 & 86.8 & 86.5 & 86.9 & 86.7 \\
\midrule
GPQA-D & 61.4 & 62.0 & 62.8 & 63.6 & 64.1 & 64.2 & 64.5 & 64.3 & 64.4 \\
\midrule
LiveCodeBench & 38.1 & 40.3 & 40.7 & 41.4 & 42.0 & 42.5 & 42.2 & 42.6 & 42.6 \\
\bottomrule
\end{tabular}
}
\end{table}

We provide detailed ablation results on the length of latent thoughts $m$ in Table~\ref{tab:ablation_latent_steps}, corresponding to the plot in Figure~\ref{fig:latent_steps}.
As $m$ increases, \ours consistently improves across all benchmarks, and the performance gradually saturates around $m=80$, suggesting that a moderate latent thought budget is sufficient for effective latent collaboration.

\newpage
\section{Prompt Template for RecursiveMAS}

\begin{tcolorbox}[
  title={Prompt Template for Sequential-Style \ours},
colframe=green!30!black,
colback=green!2!white,
colbacktitle=green!35!black!80,
coltitle=white,
borderline={0.8pt}{0pt}{green!50!black},
  fonttitle=\bfseries,
  width=\textwidth,
  left=0mm,
  right=0mm,
  boxrule=1pt,
  top=1mm,
  bottom=1mm,
  arc=1mm
]
\footnotesize
\textbf{System Prompt for All Agents:}\\
You are a helpful assistant. \\

\textbf{User Prompt for Planner Agent:}\\
You are a planner agent in a recursive multi-agent system. 
Here is the latent information from previous round: \text{\{Latent Thought Embeddings\}}.
Given the latent information, you should output a step-by-step plan to solve the question: \text{\{Question\}} \\

\textbf{User Prompt for Critic Agent:}\\
You are a critic agent in a recursive multi-agent system.
Here is the latent information from previous agent: \text{\{Latent Thought Embeddings\}}.
Given the latent information, you should critique the initial plan and output an improved plan to solve the question: \text{\{Question\}} \\

\textbf{User Prompt for Solver Agent:}\\
You are a solver agent in a recursive multi-agent system.
Here is the latent information from previous agent: \text{\{Latent Thought Embeddings\}}
Given the latent information, you should solve the question and provide the final answer: \text{\{Question\}} \\
Solve the question and put the final answer inside $\backslash$boxed\{\}.

\end{tcolorbox}

\begin{tcolorbox}[
  title={Prompt Template for Mixture-Style \ours},
colframe=green!30!black,
colback=green!2!white,
colbacktitle=green!35!black!80,
coltitle=white,
borderline={0.8pt}{0pt}{green!50!black},
  fonttitle=\bfseries,
  width=\textwidth,
  left=0mm,
  right=0mm,
  boxrule=1pt,
  top=1mm,
  bottom=1mm,
  arc=1mm
]
\footnotesize
\textbf{System Prompt for All Agents:}\\
You are a helpful assistant. \\

\textbf{User Prompt for Math Specialist Agent:}\\
You are a math specialist agent in a recursive multi-agent system.
Here is the latent information from previous round: \text{\{Latent Thought Embeddings\}}
Given the latent information, you should provide a domain-specific answer for the question: \text{\{Question\}} \\

\textbf{User Prompt for Science Specialist Agent:}\\
You are a science specialist agent in a recursive multi-agent system.
Here is the latent information from previous round: \text{\{Latent Thought Embeddings\}}
Given the latent information, you should provide a domain-specific answer for the question: \text{\{Question\}} \\

\textbf{User Prompt for Code Specialist Agent:}\\
You are a code specialist agent in a recursive multi-agent system.
Here is the latent information from previous round: \text{\{Latent Thought Embeddings\}}
Given the latent information, you should provide a domain-specific answer for the question: \text{\{Question\}} \\

\textbf{User Prompt for Summarizer Agent:}\\
You are a summarizer agent in a recursive multi-agent system.
Here is the latent information from the math specialist: \text{\{Math Specialist Latent Thought Embeddings\}}.
Here is the latent information from the code specialist: \text{\{Code Specialist Latent Thought Embeddings\}}.
Here is the latent information from the science specialist: \text{\{Science Specialist Latent Thought Embeddings\}}.
Given the latent information from all previous specialists, you should aggregate their reasoning and provide the final answer to the question: \text{\{Question\}} \\
Put the final answer inside $\backslash$boxed\{\}. \\

\end{tcolorbox}

\begin{tcolorbox}[
  title={Prompt Template for Distillation-Style \ours},
colframe=green!30!black,
colback=green!2!white,
colbacktitle=green!35!black!80,
coltitle=white,
borderline={0.8pt}{0pt}{green!50!black},
  fonttitle=\bfseries,
  width=\textwidth,
  left=0mm,
  right=0mm,
  boxrule=1pt,
  top=1mm,
  bottom=1mm,
  arc=1mm
]
\footnotesize
\textbf{System Prompt for All Agents:}\\
You are a helpful assistant. \\

\textbf{User Prompt for Expert Agent:}\\
You are an expert agent in a recursive multi-agent system.
Here is the latent information from previous round: \text{\{Latent Thought Embeddings\}}.
Given the latent information, you should provide expert-level guidance to solve the question: \text{\{Question\}} \\

\textbf{User Prompt for Learner Agent:}\\
You are a learner agent in a recursive multi-agent system.
Here is the latent information from previous agent: \text{\{Latent Thought Embeddings\}}.
Given the latent information, you should use the expert guidance to solve the question and provide the final answer: \text{\{Question\}} \\
Put the final answer inside $\backslash$boxed\{\}. \\

\end{tcolorbox}

\begin{tcolorbox}[
  title={Prompt Template for Deliberation-Style \ours},
colframe=green!30!black,
colback=green!2!white,
colbacktitle=green!35!black!80,
coltitle=white,
borderline={0.8pt}{0pt}{green!50!black},
  fonttitle=\bfseries,
  width=\textwidth,
  left=0mm,
  right=0mm,
  boxrule=1pt,
  top=1mm,
  bottom=1mm,
  arc=1mm
]
\footnotesize
\textbf{System Prompt for All Agents (with Tool-Calling Instructions):}\\
You are a helpful assistant that can solve the given question with the help of external tools, including a web search tool and a Python interpreter. For each question, you should reason carefully and use tools when they are necessary.
You can invoke the web search tool to retrieve factual information, and use the Python interpreter to perform calculations or execute code.
The search query and search result should be enclosed within \texttt{<search>} \texttt{</search>} and \texttt{<result>} \texttt{</result>} tags respectively.
The Python code and execution result should be enclosed within \texttt{<python>} \texttt{</python>} and \texttt{<result>} \texttt{</result>} tags respectively. \\

\textbf{User Prompt for Reflector Agent:}\\
You are a reflector agent in a recursive multi-agent system.
Here is the latent information from previous round: \text{\{Latent Thought Embeddings\}}.
Given the latent information, you should reflect on the current solution and provide an improved solution to the question: \text{\{Question\}} \\

\textbf{User Prompt for Tool-Caller Agent:}\\
You are a tool-caller agent in a recursive multi-agent system.
Here is the latent information from the Reflector agent: \text{\{Latent Thought Embeddings\}}.
Given the latent information, you should solve the question: \text{\{Question\}} \\
Put the final answer inside $\backslash$boxed\{\}. \\

\end{tcolorbox}

\newpage
\section{Case Study on Different Recursion Rounds}
\label{app:case_diff_recursion}

\begin{tcolorbox}[
  title={Output of \ours in Recursion Round 1 (On MATH500)},
  colframe=cyan!40!black,      %
  colback=cyan!6!white,        %
  colbacktitle=cyan!50!black!80,  %
  coltitle=white,                %
  fonttitle=\bfseries,
  width=\textwidth,
  left=0mm,
  right=0mm,
  boxrule=1pt,
  borderline={0.8pt}{0pt}{cyan!60!black},
  top=1mm,
  bottom=1mm,
  arc=1mm
]
\footnotesize
\textbf{Question:}\\
For how many positive integers $n>1$ is it true that $2^{24}$ is a perfect $n^{\text{th}}$ power?
\\

\hrule 
\vspace{8pt}
\textbf{Final Answer of \ours:}\\
First, let's express \( k \) in terms of its prime factorization. Since \( 2^{24} \) is a power of 2, \( k \) must also be a power of 2. Let \( k = 2^m \) for some integer \( m \). Then we have:
\[
2^{24} = (2^m)^n = 2^{mn}
\]
This equation tells us that \( 24 = mn \). We need to find all pairs of positive integers \( (m, n) \) such that \( mn = 24 \) and \( n > 1 \).

Let's list the factor pairs of 24:
\[
(1, 24), (2, 12), (3, 8), (4, 6), (6, 4), (8, 3), (12, 2), (24, 1)
\]
We discard the pairs where \( n = 1 \) because we are looking for \( n > 1 \). The remaining pairs are:
\[
(2, 12), (3, 8), (4, 6), (6, 4), (8, 3), (12, 2)
\]
Each pair \((m, n)\) corresponds to a valid \( n \) greater than 1. Therefore, the possible values of \( n \) are 2, 3, 4, 6, 8, and 12.

Counting these values, we find there are 6 possible values for \( n \).

Thus, the number of positive integers \( n > 1 \) for which \( 2^{24} \) is a perfect \( n \)-th power is \(\boxed{6}\). \textcolor{red}{\textit{Incorrect Answer} \ding{55}}

\end{tcolorbox}

\begin{tcolorbox}[
  title={Output of \ours in Recursion Round 2 (On MATH500)},
  colframe=cyan!40!black,      %
  colback=cyan!6!white,        %
  colbacktitle=cyan!50!black!80,  %
  coltitle=white,                %
  fonttitle=\bfseries,
  width=\textwidth,
  left=0mm,
  right=0mm,
  boxrule=1pt,
  borderline={0.8pt}{0pt}{cyan!60!black},
  top=1mm,
  bottom=1mm,
  arc=1mm
]
\footnotesize
\textbf{Question:}\\
For how many positive integers $n>1$ is it true that $2^{24}$ is a perfect $n^{\text{th}}$ power?
\\

\hrule 
\vspace{8pt}
\textbf{Final Answer of \ours:}\\
First, let's express \( 2^{24} \) in terms of its prime factorization: \( 2^{24}\)

For \( 2^{24} \) to be a perfect \( n^{\text{th}} \) power, \( n \) must be a divisor of 24. This is because if \( 2^{24} = k^n \), then \( k \) must be a power of 2, say \( k = 2^m \). Therefore, we have:
\[
2^{24} = (2^m)^n = 2^{mn}
\]

Equating the exponents, we get:
\[
24 = mn
\]

We need to find all pairs \((m, n)\) such that \( mn = 24 \) and \( n > 1 \). Let's list the divisors of 24:
\[
1, 2, 3, 4, 6, 8, 12, 24
\]

Since \( n > 1 \), we exclude 1 from the list of divisors. The remaining divisors are:
\[
2, 3, 4, 6, 8, 12, 24
\]

Thus, the number of positive integers \( n > 1 \) for which \( 2^{24} \) is a perfect \( n^{\text{th}} \) power is: \( \boxed{7} \) \textcolor{impr}{\textit{Correct Answer} \ding{51}}

\end{tcolorbox}

\begin{tcolorbox}[
  title={Output of \ours in Recursion Round 3 (On MATH500)},
  colframe=cyan!40!black,      %
  colback=cyan!6!white,        %
  colbacktitle=cyan!50!black!80,  %
  coltitle=white,                %
  fonttitle=\bfseries,
  width=\textwidth,
  left=0mm,
  right=0mm,
  boxrule=1pt,
  borderline={0.8pt}{0pt}{cyan!60!black},
  top=1mm,
  bottom=1mm,
  arc=1mm
]
\footnotesize
\textbf{Question:}\\
For how many positive integers $n>1$ is it true that $2^{24}$ is a perfect $n^{\text{th}}$ power?
\\

\hrule 
\vspace{8pt}
\textbf{Final Answer of \ours:}\\
First, let's express \( 2^{24} \) in terms of its prime factorization:

For \( 2^{24} \) to be a perfect \( n \)-th power, \( n \) must be a divisor of 24. This is because if \( 2^{24} = k^n \), then \( k \) must be of the form \( 2^m \) for some integer \( m \), and thus we have
\(
(2^m)^n = 2^{mn} = 2^{24}
\), this implies:
\[
mn = 24
\]

We need to find all positive divisors of 24. The divisors of 24 are:
\[
1, 2, 3, 4, 6, 8, 12, 24
\]

Since we are looking for \( n > 1 \), we exclude 1 from the list. The remaining divisors are:
\[
2, 3, 4, 6, 8, 12, 24
\]

Thus, the number of positive integers \( n > 1 \) for which \( 2^{24} \) is a perfect \( n \)-th power is:
\(
\boxed{7}
\) \textcolor{impr}{\textit{Correct Answer} \ding{51}}

\end{tcolorbox}

\begin{tcolorbox}[
  title={Output of TextMAS (On MATH500)},
  colframe=cyan!40!black,      %
  colback=cyan!6!white,        %
  colbacktitle=cyan!50!black!80,  %
  coltitle=white,                %
  fonttitle=\bfseries,
  width=\textwidth,
  left=0mm,
  right=0mm,
  boxrule=1pt,
  borderline={0.8pt}{0pt}{cyan!60!black},
  top=1mm,
  bottom=1mm,
  arc=1mm
]
\footnotesize
\textbf{Question:}\\
For how many positive integers $n>1$ is it true that $2^{24}$ is a perfect $n^{\text{th}}$ power?
\\

\hrule 
\vspace{8pt}
\textbf{Output of Planner Agent:}\\
Ok, I need to provide a step-by-step plan to solve the problem.
\begin{itemize}
\item Step 1: Determine the prime factorization of $2^{24}$.

\item Step 2: Recognize that $2^{24}$ can be expressed as a perfect $n^{\text{th}}$ power if the exponent $24$ is divisible by $n$.

\item Step 3: Find all positive integers $n > 1$ such that $24$ is divisible by $n$.

\item Step 4: Count the number of such integers $n$.

\item Step 5: Conclude that the number of positive integers $n > 1$ for which $2^{24}$ is a perfect $n^{\text{th}}$ power is equal to the number of divisors of $24$ that are greater than 1.

\item Step 6: List the divisors of $24$ greater than 1 and count them.

\item Step 7: The final answer is the count of such divisors.
\end{itemize}

\hrule 
\vspace{8pt}
\textbf{Output of Critic Agent:}\\
Here is a detailed plan for the problem:
\begin{description}
    \item[\textmd{Step 1: Analyze the problem}] \hfill \\
    We are asked to find how many positive integers $n > 1$ satisfy the condition that $2^{24}$ is a perfect $n^{\text{th}}$ power.

    \item[\textmd{Step 2: Recall the definition of a perfect $n^{\text{th}}$ power}] \hfill \\
    A number $x$ is a perfect $n^{\text{th}}$ power if there exists an integer $k$ such that $x = k^n$.

    \item[\textmd{Step 3: Express $2^{24}$ as a power of some base}] \hfill \\
    We can write $2^{24} = (2^m)^n$, where $m$ and $n$ are positive integers.

    \item[\textmd{Step 4: Find the possible values of $n$}] \hfill \\
    For $2^{24}$ to be a perfect $n^{\text{th}}$ power, $n$ must be a divisor of 24 (since $2^{24} = (2^m)^n$ implies $n$ divides 24).

    \item[\textmd{Step 5: Count the number of valid divisors}] \hfill \\
    Find the total number of positive divisors of 24, and then subtract 1 to account for the restriction that $n > 1$.
\end{description}

\hrule 
\vspace{8pt}
\textbf{Output of Solver Agent:}\\
To determine for how many positive integers \( n > 1 \) the number \( 2^{24} \) is a perfect \( n \)-th power, we need to find all the divisors of 24, since \( 2^{24} \) can be written as \[ (2^k)^n \], where \( kn = 24 \).
First, we list all the divisors of 24. The divisors of 24 are:
\[ 1, 2, 3, 4, 6, 8, 12, 24 \]

Each divisor \( d \) of 24 corresponds to a possible value of \( n \) such that \( 2^{24} \) is a perfect \( n \)-th power. Specifically, if \( d \) is a divisor of 24, then \( n = \frac{24}{d} \) will be an integer and \( 2^{24} = (2^{d})^{24/d} \).

Let's list the corresponding \( n \) values for each divisor \( d \):
- If \( d = 1 \), then \( n = \frac{24}{1} = 24 \)
- If \( d = 2 \), then \( n = \frac{24}{2} = 12 \)
- If \( d = 3 \), then \( n = \frac{24}{3} = 8 \)
- If \( d = 4 \), then \( n = \frac{24}{4} = 6 \)
- If \( d = 6 \), then \( n = \frac{24}{6} = 4 \)
- If \( d = 8 \), then \( n = \frac{24}{8} = 3 \)
- If \( d = 12 \), then \( n = \frac{24}{12} = 2 \)
- If \( d = 24 \), then \( n = \frac{24}{24} = 1 \) (but since \( n > 1 \), we exclude this case)

Thus, the possible values of \( n \) are \( 2, 3, 4, 6, 8, 12, 24 \). There are 7 such values.

Therefore, the number of positive integers \( n > 1 \) for which \( 2^{24} \) is a perfect \( n \)-th power is:
\( \boxed{7} \) \textcolor{impr}{\textit{Correct Answer} \ding{51}}

\end{tcolorbox}

\section{Examples of RecursiveMAS Across Different Downstream Tasks}
\label{app:case_study_question}

\begin{tcolorbox}[
  title={Case Study of \ours on Mathematical Reasoning Questions (AIME2026)},
  colframe=orange!60!black,     %
  colback=orange!5,             %
  colbacktitle=orange!70!black!80, %
  coltitle=white,               %
  fonttitle=\bfseries,
  width=\textwidth,
  left=0mm,
  right=0mm,
  boxrule=1pt,
  borderline={0.8pt}{0pt}{orange!50!black},
  top=1mm,
  bottom=1mm,
  arc=1mm
]
\footnotesize
\textbf{Question:}\\
A plane contains points $A$ and $B$ with $AB = 1$. Point $A$ is rotated in the plane counterclockwise through an acute angle $\theta$ around point $B$ to point $A^\prime$. Then $B$ is rotated in the plane clockwise through angle $\theta$ around point $A^\prime$ to point $B^\prime$. Suppose that $AB^\prime = \frac{4}{3}$. The value of $\cos \theta$ can be written as $\frac{m}{n}$ , where $m$ and $n$ are relatively prime positive integers. Find $m + n$.\\

\hrule 
\vspace{8pt}
\textbf{Recursion Round 1 (in latent space):}\\
\textbf{
\textit{Planner} $\quad\rightarrow\quad$ 
\textit{Critic} $\quad\rightarrow\quad$ 
\textit{Solver} $\quad\rightarrow\quad$ } \\

\hrule 
\vspace{8pt}

\textbf{Recursion Round 2 (in latent space):}\\
\textbf{
\textit{Planner} $\quad\rightarrow\quad$ 
\textit{Critic} $\quad\rightarrow\quad$ 
\textit{Solver} $\quad\rightarrow\quad$ } \\

\hrule 
\vspace{8pt}

\textbf{Recursion Round 3 (with final output):}\\
\textbf{
\textit{Planner} $\quad\rightarrow\quad$ 
\textit{Critic} $\quad\rightarrow\quad$ 
\textit{Solver}: }\\
To solve the problem, we need to analyze the geometric transformations and use the given distances to find the value of $\cos \theta$.

\begin{enumerate}[leftmargin=*]
    \item \textbf{Initial Setup:}
    Points $A$ and $B$ are in the plane with $AB = 1$. Point $A$ is rotated counterclockwise through an acute angle $\theta$ around point $B$ to point $A'$. Point $B$ is then rotated clockwise through the same angle $\theta$ around point $A'$ to point $B'$. We are given that $AB' = \frac{4}{3}$.

    \item \textbf{Coordinate System:}
    Place point $B$ at the origin $(0, 0)$ and point $A$ at $(1, 0)$.

    \item \textbf{Rotation of Point $A$:}
    When point $A$ is rotated counterclockwise by angle $\theta$ around point $B$, its new coordinates $A'$ are:
    $$
    A' = (1 \cos \theta, 1 \sin \theta) = (\cos \theta, \sin \theta).
    $$

    \item \textbf{Rotation of Point $B$:}
    Point $B$ is rotated clockwise by angle $\theta$ around point $A'$. The coordinates of $B'$ after this rotation can be found using:
    $$
    B' = A' + R_{-\theta}(B - A'),
    $$
    where $R_{-\theta}$ is the rotation matrix for clockwise rotation by $\theta$:
    $$
    R_{-\theta} = \begin{pmatrix}
    \cos \theta & \sin \theta \\
    -\sin \theta & \cos \theta
    \end{pmatrix}.
    $$
    The coordinates of $B - A'$ are $B - A' = (0 - \cos \theta, 0 - \sin \theta) = (-\cos \theta, -\sin \theta)$. Applying the rotation matrix:
    $$
    R_{-\theta}(B - A') = \begin{pmatrix}
    \cos \theta & \sin \theta \\
    -\sin \theta & \cos \theta
    \end{pmatrix} \begin{pmatrix}
    -\cos \theta \\
    -\sin \theta
    \end{pmatrix} = \begin{pmatrix}
    -\cos^2 \theta - \sin^2 \theta \\
    \cos \theta \sin \theta - \sin \theta \cos \theta
    \end{pmatrix} = \begin{pmatrix}
    -1 \\
    0
    \end{pmatrix}.
    $$
    Therefore, the coordinates of $B'$ are:
    $$
    B' = (\cos \theta, \sin \theta) + (-1, 0) = (\cos \theta - 1, \sin \theta).
    $$

    \item \textbf{Distance $AB'$:}
    The distance $AB'$ is given by:
    $$
    AB' = \sqrt{(\cos \theta - 1 - 1)^2 + (\sin \theta - 0)^2} = \sqrt{(\cos \theta - 2)^2 + \sin^2 \theta}.
    $$
    Simplifying inside the square root:
    $$
    AB' = \sqrt{\cos^2 \theta - 4 \cos \theta + 4 + \sin^2 \theta} = \sqrt{1 - 4 \cos \theta + 4} = \sqrt{5 - 4 \cos \theta}.
    $$
    We are given that $AB' = \frac{4}{3}$, so $\sqrt{5 - 4 \cos \theta} = \frac{4}{3}$. Squaring both sides gives $5 - 4 \cos \theta = \frac{16}{9}$. Solving for $\cos \theta$:
    $$
    5 - \frac{16}{9} = 4 \cos \theta \implies \frac{45}{9} - \frac{16}{9} = 4 \cos \theta \implies \frac{29}{9} = 4 \cos \theta \implies \cos \theta = \frac{29}{36}.
    $$

    \item \textbf{Final Answer:}
    The value of $\cos \theta$ is $\frac{29}{36}$. The sum of the numerator and the denominator is $m + n = 29 + 36 = 65$.
\end{enumerate}

Thus, the final answer is $\boxed{65}$. \textcolor{impr}{\textit{Correct Answer} \ding{51}}

\end{tcolorbox}

\begin{tcolorbox}[
  title={Case Study of \ours on Scientific and Medical Questions (AIME2026)},
  colframe=orange!60!black,     %
  colback=orange!5,             %
  colbacktitle=orange!70!black!80, %
  coltitle=white,               %
  fonttitle=\bfseries,
  width=\textwidth,
  left=0mm,
  right=0mm,
  boxrule=1pt,
  borderline={0.8pt}{0pt}{orange!50!black},
  top=1mm,
  bottom=1mm,
  arc=1mm
]
\footnotesize
\textbf{Question:}\\
A 45-year-old mechanic presents to the emergency department complaining of acute-onset shortness of breath while repairing a plowing tractor for his neighbor. The patient denies having any history of asthma or respiratory symptoms, and does not smoke. His temperature is 99.8°F (37.7°C), pulse is 65/min, blood pressure is 126/86 mmHg, and respirations are 20/min. His oxygen saturation is 97\%. On exam, he is pale and diaphoretic. His pupils are contracted. Diffuse wheezes are noted in all lung fields. What is the best treatment for his condition?\\
A. Succinylcholine\\
B. Inhaled ipratropium and oxygen\\
C. Atropine and pralidoxime\\
D. Inhaled albuterol and oxygen
\\

\hrule 
\vspace{8pt}
\textbf{Recursion Round 1 (in latent space):}\\
\textbf{
\textit{Planner} $\quad\rightarrow\quad$ 
\textit{Critic} $\quad\rightarrow\quad$ 
\textit{Solver} $\quad\rightarrow\quad$ } \\

\hrule 
\vspace{8pt}

\textbf{Recursion Round 2 (in latent space):}\\
\textbf{
\textit{Planner} $\quad\rightarrow\quad$ 
\textit{Critic} $\quad\rightarrow\quad$ 
\textit{Solver} $\quad\rightarrow\quad$ } \\

\hrule 
\vspace{8pt}

\textbf{Recursion Round 3 (with final output):}\\
\textbf{
\textit{Planner} $\quad\rightarrow\quad$ 
\textit{Critic} $\quad\rightarrow\quad$ 
\textit{Solver}: }\\
The clinical presentation is highly suggestive of organophosphate poisoning. 
Clinical Signs:
\begin{itemize}[leftmargin=*]
    \item \textbf{Dizziness/Shortness of breath}: Suggests respiratory distress.
    \item \textbf{Pupils are contracted (Miosis)}: This is a classic sign of cholinergic excess (SLUDGE mnemonic: Salivation, Lacrimation, Urination, Defecation, GI upset, Emesis).
    \item \textbf{Diffuse wheezes}: Caused by bronchorrhea and bronchospasm due to overstimulation in the lungs.
    \item \textbf{Pale and diaphoretic}: Diaphoresis is a key sign of cholinergic crisis.
    \item \textbf{Vitals}: The patient is tachypneic (20/min is borderline but in context of distress, this is significant) and hypoxic (implied by distress, though 97\% is maintained, the saturation will drop quickly without treatment).
\end{itemize}

Analysis of Options:
\begin{itemize}[leftmargin=*]
    \item \textbf{A. Succinylcholine}: Depolarizing neuromuscular blocker. Exacerbates or prolongs paralysis due to high endogenous acetylcholine; does not resolve the underlying cholinergic crisis.
    \item \textbf{B. Inhaled ipratropium and oxygen}: Anticholinergic. Relieves local respiratory symptoms but fails to address the root cause and systemic cholinergic toxicity.
    \item \textbf{C. Atropine and pralidoxime}: \textbf{Correct treatment}. Atropine (muscarinic antagonist) is life-saving, reversing severe secretions and bronchospasm. Pralidoxime reactivates acetylcholinesterase to resolve nicotinic effects.
    \item \textbf{D. Inhaled albuterol and oxygen}: Beta-2 agonist. May help bronchospasm but ineffective against massive bronchorrhea and systemic poisoning.
\end{itemize}

\textbf{Conclusion}:
The patient is suffering from organophosphate poisoning (cholinergic crisis). The standard of care involves the administration of atropine to reverse muscarinic effects (especially respiratory secretions and bronchospasm) and pralidoxime to regenerate acetylcholinesterase. So the final answer is:  
$\boxed{C}$ \textcolor{impr}{\textit{Correct Answer} \ding{51}}

\end{tcolorbox}

\begin{tcolorbox}[
  title={Case Study of \ours on Code Generation Tasks (MBPP Plus)},
  colframe=orange!60!black,     %
  colback=orange!5,             %
  colbacktitle=orange!70!black!80, %
  coltitle=white,               %
  fonttitle=\bfseries,
  width=\textwidth,
  left=0mm,
  right=0mm,
  boxrule=1pt,
  borderline={0.8pt}{0pt}{orange!50!black},
  top=1mm,
  bottom=1mm,
  arc=1mm
]
\footnotesize
\textbf{Question:}\\
Write a function to check if each element of second tuple is smaller than its corresponding element in the first tuple. \\

\hrule 
\vspace{8pt}
\textbf{Recursion Round 1 (in latent space):}\\
\textbf{
\textit{Planner} $\quad\rightarrow\quad$ 
\textit{Critic} $\quad\rightarrow\quad$ 
\textit{Solver} $\quad\rightarrow\quad$ } \\

\hrule 
\vspace{8pt}

\textbf{Recursion Round 2 (in latent space):}\\
\textbf{
\textit{Planner} $\quad\rightarrow\quad$ 
\textit{Critic} $\quad\rightarrow\quad$ 
\textit{Solver} $\quad\rightarrow\quad$ } \\

\hrule 
\vspace{8pt}

\textbf{Recursion Round 3 (with final output):}\\
\textbf{
\textit{Planner} $\quad\rightarrow\quad$ 
\textit{Critic} $\quad\rightarrow\quad$ 
\textit{Solver}: }\\
To solve this problem, we need to write a function that takes two tuples as input and checks if each element of the second tuple is smaller than its corresponding element in the first tuple. If all elements satisfy this condition, the function should return True; otherwise, it should return False.

Here is the code logic:
\begin{enumerate}[leftmargin=*]
    \item Define the function `check\_smaller` that takes two tuples as input.
    \item Check if the lengths of the two tuples are equal. If they are not, return False.
    \item Iterate over the elements of the tuples and check if each element of the second tuple is smaller than its corresponding element in the first tuple.
    \item If all elements satisfy the condition, return True; otherwise, return False.
\end{enumerate}
Here is the Python code for the solution:

\vspace{-10pt}
\begin{verbatim}
def check_smaller(tuple1, tuple2):
    # Check if the lengths of the tuples are equal
    if len(tuple1) != len(tuple2):
        return False
    
    # Iterate over the elements of the tuples
    for i in range(len(tuple1)):
        if tuple1[i] <= tuple2[i]:
            return False
    
    # If all elements satisfy the condition, return True
    return True
\end{verbatim}\\
\textcolor{impr}{\textit{Correct Answer} \ding{51}}
\end{tcolorbox}

\begin{tcolorbox}[
  title={Case Study of \ours on Search Based Tasks (Bamboogle)},
  colframe=orange!60!black,     %
  colback=orange!5,             %
  colbacktitle=orange!70!black!80, %
  coltitle=white,               %
  fonttitle=\bfseries,
  width=\textwidth,
  left=0mm,
  right=0mm,
  boxrule=1pt,
  borderline={0.8pt}{0pt}{orange!50!black},
  top=1mm,
  bottom=1mm,
  arc=1mm
]
\footnotesize
\textbf{Question:}\\
Who was president of the United States in the year that Citibank was founded? \\

\hrule 
\vspace{8pt}
\textbf{Recursion Round 1 (in latent space):}\\
\textbf{
\textit{Planner} $\quad\rightarrow\quad$ 
\textit{Critic} $\quad\rightarrow\quad$ 
\textit{Solver} $\quad\rightarrow\quad$ } \\

\hrule 
\vspace{8pt}

\textbf{Recursion Round 2 (in latent space):}\\
\textbf{
\textit{Planner} $\quad\rightarrow\quad$ 
\textit{Critic} $\quad\rightarrow\quad$ 
\textit{Solver} $\quad\rightarrow\quad$ } \\

\hrule 
\vspace{8pt}

\textbf{Recursion Round 3 (with final output):}\\
\textbf{
\textit{Planner} $\quad\rightarrow\quad$ 
\textit{Critic} $\quad\rightarrow\quad$ 
\textit{Solver}: }\\
To answer this question, I need to find out the year Citibank was founded and then identify the President of the United States in that year.\\
First, I will search for the founding year of Citibank.\\

\texttt{<search>}When was Citibank founded?\texttt{</search>}\\
\texttt{<result>} \\
{}[1] Citigroup - Wikipedia\\
Citibank was chartered by the State of New York on June 16, 1812, with \$2 million of capital. $\cdots$\\
{}[2] What is Brief History of Citi Company? - Matrix BCG\\
The genesis of Citi, originally named the City Bank of New York, dates back to June 16, 1812, when it received its charter in New York City. The bank commenced operations on September 14, 1812. $\cdots$\\
{}[3] Citibank - Students | Britannica Kids | Homework Help\\
Second largest bank in U.S., with about 270 branch banks in New York City area. Present name since 1976. $\cdots$\\
{}[4] A bank with history and tradition - Citi.com\\
Citibank opened for business in 1812 in New York under its original name City Bank of New York. Since then, Citi has been a major player in some of the world's most innovative breakthroughs in banking. $\cdots$\\
\texttt{</result>}\\

Based on the search results, Citibank (originally City Bank of New York) was founded in \textbf{1812}. Specifically, it was chartered on June 16, 1812, and opened for business on September 14, 1812.
Now I need to find out who was the President of the United States in 1812.\\

\texttt{<search>}President of the United States 1812\texttt{</search>}\\
\texttt{<result>} \\
{}[1] The Election of 1812 | American Battlefield Trust\\
On October 14, 1812, incumbent president, James Madison, wrote to his friend and former president, $\cdots$\\
{}[2] The War of 1812: The Nation Braces for War\\
James Madison (1751-1836) James Madison, the Father of the Constitution and the Bill of Rights, was President of the United States during the War of 1812. $\cdots$\\
{}[3] Key American Personalities | 1812 Virtual Exhibition\\
As secretary of state under President Thomas Jefferson, future president James Madison had been deeply involved in the events leading to the war. $\cdots$\\
{}[4] James Madison - Wikipedia\\
Madison was elected president in 1808. Motivated by a desire to acquire land held by Britain, Spain, and Native Americans, Madison led the United States into the War of 1812. $\cdots$\\
\texttt{</result>}\\

The search results indicate that James Madison was the President of the United States in 1812.
So, the answer is $\boxed{James\ Madison}$ \textcolor{impr}{\textit{Correct Answer} \ding{51}}

\end{tcolorbox}

\end{document}